\documentclass[twoside,11pt]{article}
\usepackage{jair, theapa, rawfonts}
\jairheading{78}{2023}{709-746}{09/2022}{11/2023}

\usepackage[usenames,dvipsnames]{xcolor}
\usepackage{amsmath,amssymb,amsthm}
\usepackage[ruled,linesnumbered]{algorithm2e}
\usepackage{thmtools} 
\usepackage{booktabs}
\usepackage{paralist}
\usepackage{caption}
\usepackage{subcaption}
\usepackage{lscape}
\usepackage{graphicx}
\usepackage{adjustbox}
\usepackage{multirow,array}
\usepackage[normalem]{ulem}
\usepackage{enumitem}
\usepackage[capitalise]{cleveref}

\usepackage{cancel}

\newcommand\citet[1]{\citeauthor{#1}~\citeyear{#1}}

\newcommand\m[1]{\mathcal{#1}}

\DeclareMathOperator*{\mini}{minimize}

\newtheorem{thm}{Theorem}
\newtheorem{definition}[thm]{Definition}

\newtheorem{propo}{Proposition}
\newtheorem{proposition}[propo]{Proposition}

\newtheorem{example}{Example}
\newtheorem{obs}{Observation}
\newtheorem{observation}[obs]{Observation}
\newtheorem{theorem}[thm]{Theorem}

\newcommand\setstohit{\ensuremath{\m{H} }\xspace}
\newcommand\decisionvariables{{D}}
\newcommand\decisionvariable{{d}}
\newcommand\settohit{H}
\newcommand\hittingset{h}
\newcommand\F{\ensuremath{\m{F} }\xspace}
\newcommand\subsetT{\ensuremath{\m{S} }\xspace}

\newcommand\Iend{\ensuremath{I_\mathit{end} }\xspace}
\newcommand\formula{\ensuremath{\m{F} }\xspace}
\newcommand\formulac{\ensuremath{\m{C} }\xspace}
\newcommand\formularunning{\ensuremath{C }\xspace}

\newcommand\mm[1]{\ensuremath{#1}\xspace}
\newcommand\nat{\mm{\mathbb{N}}}
\newcommand\ltrue{\mm{\textbf{t}}}

\newcommand\call[1]{\mm{\texttt{#1}}}
\newcommand\cohs{\mm{\call{CondOptHittingSet}}}

\newcommand\sat{\texttt{{SAT}}\xspace}
\newcommand\subsetmaxsat{\texttt{SubsetMax-SAT}}
\newcommand\grow{\mm{\call{Grow}}}
\newcommand\omus{\mm{\call{OUS}}}
\newcommand\comus{\mm{\call{OCUS}}}
\newcommand\mus{\mm{\call{MUS}}}
\newcommand{\comusbound}{\mm{\call{OCUS\_Bound}}}
\newcommand{\comussplit}{\mm{\call{OCUS\_Split}}}
\newcommand\comusincr{\mm{\call{OCUS+Incr.}}}
\newcommand\ousb{\mm{\call{OCUS\_Bound}}}
\newcommand\ousbiter{\mm{\call{OCUS\_Split}}}
\newcommand\ousbincr{\mm{\call{OCUS\_Bound+Incr.}}}
\newcommand\ousbiterincr{\mm{\call{OCUS\_Split+Incr.}}}
\newcommand\corrss{\mm{\call{CorrSS}}}
\newcommand\corrsubsets{\mm{\call{CorrSubsets}}}

\newcommand\hitsetbased{hitting set--based\xspace} 
\newcommand\satsets{\mm{\mathbf{SSs}}}

\newcommand\maxsat{\texttt{MaxSAT}}
\newcommand\domspecmaxsat{\texttt{Dom.-spec.}~\texttt{MaxSAT}}
\newcommand\maxsatfull{\texttt{MaxSAT Full}}
\newcommand\voc{\ensuremath{\Sigma}\xspace}
\newcommand\negset[1]{\mm{\overline{#1}}}

\newcommand\muses[1]{\ensuremath{\mathit{MUSs}(#1)}\xspace}
\newcommand\mcses[1]{\ensuremath{\mathit{MCSs}(#1)}\xspace}

\newcommand\multisubsetmaxsat{\emph{Multi}-\subsetmaxsat\xspace}
\newcommand\onestep{\ensuremath{\call{Explain-One-Step}}\xspace}

\newcommand\onestepomusbounded{\ensuremath{\call{Explain-One-Step-OCUS-Bounded}}\xspace}

\newcommand\onestepo{\ensuremath{\call{Explain-One-Step-OCUS}}\xspace}
\newcommand\iterativeonestep{\ensuremath{\call{Explain-One-Step-OCUS-Split}}\xspace}

\ShortHeadings{Efficiently Explaining CSPs with Unsatisfiable Subset Optimization}{Gamba, Bogaerts, \& Guns}
\firstpageno{709}

\begin{document}
	
\SetKwInOut{Input}{Input}
\SetKwInOut{OptInput}{Optional}
\SetKwInOut{Output}{Output}
\SetKwInOut{State}{State}
\SetKwInOut{Ext.}{Ext}
\SetKwComment{command}{/*}{*/}
\SetKw{Break}{break}

\title{Efficiently Explaining CSPs with Unsatisfiable\\Subset Optimization}

\author{\name Emilio Gamba \email emilio.gamba@kuleuven.be \\
\name Bart Bogaerts \email bart.bogaerts@vub.be \\
\addr Vrije Universiteit Brussel, Pleinlaan 5\\
1050 Elsene, Belgium
\AND
\name Tias Guns \email tias.guns@kuleuven.be \\
\addr KULeuven, Oude Markt 13 - bus 5500\\
3000 Leuven, Belgium
}

\maketitle

\begin{abstract}
	We build on a recently proposed method for stepwise explaining the solutions to Constraint Satisfaction Problems (CSPs) in a human understandable way. 
An explanation here is a sequence of simple inference steps where simplicity is quantified by a cost function. 
Explanation generation algorithms rely on extracting Minimal Unsatisfiable Subsets (MUSs) of a derived unsatisfiable formula, exploiting a one-to-one correspondence between so-called \emph{non-redundant explanations} and MUSs. 
However, MUS extraction algorithms do not guarantee subset minimality or optimality with respect to a given cost function. 
Therefore, we build on these formal foundations and address the main points of improvement, namely how to generate explanations \emph{efficiently} that are provably \emph{optimal} (with respect to the given cost metric). 
To this end, we developed
(1) a hitting set-based algorithm for finding the optimal constrained unsatisfiable subsets;
(2) a method for reusing relevant information across multiple algorithm calls; and
(3) methods for exploiting domain-specific information to speed up the generation of explanation sequences.
We have experimentally validated our algorithms on a large number of CSP problems. We found that our algorithms outperform the MUS approach in terms of \emph{explanation quality} and \emph{computational time} (on average up to 56 \% faster than a standard MUS approach).
\end{abstract}

\section{Introduction}
\label{sec:introduction}

Building on old ideas to explain domain-specific propagation performed by constraint solvers \cite{sqalli1996inference,freuder2001explanation}, we recently introduced a method that takes as input a \emph{satisfiable} set of constraints
and explains the solution finding process in a human understandable way  \shortcite{ecai/BogaertsGCG20}. 
Explanations in this work are sequences of simple inference steps that involve as few constraints and previously derived facts as possible. 
Each explanation step derives at least one new fact implied by a combination of constraints and previously derived facts.

The explanation steps of \citet{ecai/BogaertsGCG20} are heuristically optimised with respect to a given cost function that should approximate human understandability, e.g. taking into account the number of constraints and facts considered, as well as an estimate of their cognitive complexity.
For example, when evaluating explanations for logic grid puzzles, the given clues of the puzzle are considered more difficult than simpler reasoning tricks, that are present in all instances of such puzzles and therefore have a higher cost.

In practice, the explanation generation algorithms presented in \citet{ecai/BogaertsGCG20} rely heavily on calls to \emph{Minimal Unsatisfiable Subsets} (MUS) \cite{marques2010minimal} of a derived unsatisfiable formula, exploiting a one-to-one correspondence between so-called \emph{non-redundant explanations} and MUSs. 

However, the algorithm of \citet{ecai/BogaertsGCG20} has two main weaknesses.
First, it provides \emph{no guarantees on the quality} of the generated explanations due to internally relying on the computation of $\subseteq$-minimal unsatisfiable subsets, which are often suboptimal with respect to the given cost function. 
Second, it suffers from \emph{performance problems}: the lack of optimality is partly overcome by calling a MUS algorithm on increasingly larger subsets of constraints for each candidate fact to explain.
However, using multiple MUS calls per literal in each iteration quickly leads to efficiency problems, causing the explanation generation process to take several hours, even for simple puzzles designed to be solvable by humans.

\paragraph{Contributions} In this paper, we tackle the limitations discussed above.
We develop algorithms that aid explaining in Constraint Satisfaction Problems and improve the state of the-art in the following ways: 

\begin{itemize}
\item We develop algorithms that compute (cost-)\textbf{Optimal} Unsatisfiable Subsets (from now on called OUSs) based on the well-known hitting-set duality which is also used to compute cardinality-minimal MUSs \shortcite{ignatiev2015smallest,DBLP:conf/kr/SaikkoWJ16}.

\item We observe that in the explanation setting, many of the individual calls for MUSs (or OUSs) can actually be replaced by a single call that searches for an optimal unsatisfiable subset \textbf{among subsets satisfying certain structural constraints}. 
We formalize and generalize this observation by introducing the \emph{Optimal \textbf{Constrained} Unsatisfiable Subsets (OCUS)} problem. We then show how $O(n^2)$ calls to MUS/OUS can be replaced by $O(n)$ calls to an OCUS oracle, where $n$ denotes the number of facts to explain. 

\item We develop techniques for further \textbf{optimizing} the O(C)US algorithms, exploiting domain-specific information coming from the fact that we are in the \emph{explanation-generation context}. Such optimizations include 
\begin{inparaenum}[(i)]
	\item the development of methods for \textbf{information re-use} between consecutive O(C)US calls; as well as 
	\item an explanation-specific version of the OCUS algorithm.
\end{inparaenum} 

\item Finally, we extensively \textbf{evaluate} our approaches on a large number of CSP problems.

\end{itemize}

\paragraph{Paper Structure} The rest of this paper is structured as follows. 
In \cref{sec:related-works}, we discuss work related to efficiently computing unsatisfiable subsets. \cref{sec:background} introduces the theoretical foundations of our implicit hitting set approach. \cref{sec:motivation} motivates our work, and \cref{sec:OCUS} introduces the OCUS problem and a hitting set-based algorithm for computing OCUSs. Then, in \cref{sec:OUS}, we look at 2 other methods for computing the next best explanation in the explanation sequence. In \cref{sec:efficient-ocus} we present methods to improve the efficiency of our algorithms, e.g. by making them incremental or by using different methods to extract multiple correction subsets.
Finally, in \cref{sec:experiments} we evaluate our approach on a large set of puzzle instances and conclude with some future perspectives.

\paragraph*{Publication History} This article is an extension of a previous paper presented at the International Joint Conference on Artificial Intelligence 2021 \cite{ijcai2021}. The current paper extends the previous paper with more detailed examples, extensive experiments on a large data set of puzzles, as well as the novel \iterativeonestep algorithm for efficiently computing explanations.

\section{Related Work}\label{sec:related-works}

In the last few years, driven by the increasing many successes in Artificial Intelligence (AI), there has been a growing need for \textbf{eXplainable Artificial Intelligence (XAI)}~\cite{miller2019explanation}.
In the research community, this need manifests itself through the emergence of (interdisciplinary) workshops and conferences on this topic~\shortcite{xai-ijcai,FAT}, as well as American and European incentives to stimulate research in the area~\shortcite{gunning2017explainable,hamonrobustness,fetproact}. 

While the main focus of XAI research has been on explaining black-box machine learning systems \shortcite{lundberg2017unified,guidotti2018survey,ignatiev2019abduction}. Model-based systems, which are typically considered more transparent, are also in need of explanation mechanisms. 
For instance, \citet{vassiliades2021argumentation} surveys the important methods that use argumentation \shortcite{modgil2013added} to provide explainability in AI, with for example applications in medical diagnosis \shortcite{obeid2019using}. Abstract argumentation frameworks introduce an abstract formalism to explain argumentative acceptance \cite{vsevselja2013abstract,liao2020explanation,ulbricht2021strong}. Description Logics \cite{baader2004description} on the other hand, aim at explaining logical proofs \shortcite{alrabbaa2021finding,koopmann2021two}, i.e `why does phi follow from psi?'. 

The main focus of our work lies in providing explainable agency \cite{langley2017explainable} to Constraint Programming (CP) \cite{fai/Rossi06} and Boolean Satisfiability (SAT) \shortcite{faia/2009-185} systems.
Advances in solving, as well as hardware improvements, allow for such systems now easily consider millions of alternatives in short amounts of time.
Because of this complexity, the question arises of how to generate human-interpretable explanations of the conclusions they make. 
Explanations have seen a rejuvenation in various subdomains of constraint reasoning \shortcite{fox2017explainable,vcyras2019argumentation,chakraborti2017plan,ecai/BogaertsGCG20}.
In the CP and SAT communities, there has been a strong focus on explaining \textit{unsatisfiable} problem instances~\cite{junker2001quickxplain}, for example by extracting a minimal unsatisfiable subset (MUS) from the problem constraints \cite{DBLP:journals/jar/LiffitonS08}. 

We have recently introduced \emph{step-wise explanations} \cite{ecai/BogaertsGCG20} to explain the solution of satisfiable instances, exploiting a one-to-one correspondence between so-called non-redundant explanations and MUSs of a derived program. The focus of \citet{ecai/BogaertsGCG20} was on explaining Zebra puzzles; building on this work, \citet{schotten} investigated interpretable explanations using MUSs for a wide range of puzzles.  
Our current work is motivated by a concrete algorithmic need: to generate these explanations \emph{efficiently}, we need algorithms that can find MUSs that are optimal with respect to a given cost function, where the cost function approximates human-understandability of the corresponding explanation step. 

The closest related work can be found in the literature on generating or enumerating MUSs \cite{conf/sat/LynceM04,bacchus2016finding,bacchus2015using,liffiton2016fast}. 
Various techniques are used to find MUSs, including manipulation of resolution proofs produced by SAT solvers \shortcite{goldberg,DBLP:journals/fmsd/GershmanKS08,DBLP:conf/sat/DershowitzHN06}, incremental solving to enable/disable clauses and branch-and-bound search \shortcite{DBLP:conf/dac/OhMASM04}, BDD-manipulation methods \cite{huang}. 
Some methods rely on \emph{seed-shrink} algorithms \cite{bendik2020must,bendik2020replication} which repeatedly start from unsatisfiable \emph{seed} (an unsatisfiable core) and \emph{shrink} the seed to a MUS.
Other methods work by means of translation into a so-called Quantified \maxsat~\cite{DBLP:journals/constraints/IgnatievJM16}, a field that combines the expressiveness of Quantified Boolean Formulas (QBF) \cite{QBF} with techniques from Maximum Satisfiability (\maxsat)~\cite{DBLP:series/faia/LiM09}, or by exploiting the so-called hitting set duality \cite{ignatiev2015smallest}.
To the best of our knowledge, only few works have considered \emph{optimizing} MUSs: the only criterion considered so far is cardinality-minimality \cite{conf/sat/LynceM04,ignatiev2015smallest}.

An \textit{abstract} framework for describing \hitsetbased algorithms, including optimization, has been developed by \citet{DBLP:conf/kr/SaikkoWJ16}. While our approach can be seen to fit within the framework, the terminology is focused on \maxsat~rather than MUS and would complicate our exposition.

\section{Background}\label{sec:background}

In this section, we introduce the terminology and concepts related to explanation generation.
We present all methods using propositional logic but our results easily generalize to richer languages, such as constraint languages, as long as the semantics is given in terms of a satisfaction relation between expressions in the language and possible states of affairs (assignments of values to variables). 

Let \voc be a set of propositional symbols, also called \emph{atoms}; this set is implicit in the rest of the paper. A \emph{literal} is an atom $p$ or its negation $\lnot p$. 
A clause is a disjunction of literals, i.e. $c_1 = x_1 \vee \lnot x_2 $. A formula $\formula$ in conjunctive normal form (CNF) is a conjunction of clauses. 
Slightly abusing notation, a clause is also viewed as a set of literals and a formula as a set of clauses. 

A (partial) interpretation $\m{I}$ is a consistent (not containing both $p$ and $\lnot p$) set of literals. 
Satisfaction of a formula \formula by an interpretation is defined as usual~\cite{faia/2009-185}: 
an interpretation $\m{I}$ \emph{satisfies} \formula if $\m{I}$ contains at least one literal from each clause in $\formula$.
A \emph{model} of \formula is an interpretation that satisfies \formula; $\formula$ is said to be \emph{unsatisfiable} if it has no models.
An interpretation that satisfies $\formula$ is also called a \emph{model} of \formula. 
A formula \formula is \emph{satisfiable} if it has at least one model and \emph{unsatisfiable} otherwise. 

A literal $l$ is a \emph{consequence} of a formula \formula if $l$ holds in all $\formula$'s models. The \emph{maximal consequence} of a formula \formula, denotes a consequence that holds in all models.
If $I$ is a set of literals, we write $\mathbf{\overline{ \raisebox{0pt}[1.1\height]{\textit{I}}}}$  for the set of negated literals $\{\lnot \ell \mid \ell \in I\}$.

\begin{example}\label{example:intro1}
	Let $\formulac_1$ be the CNF formula over atoms $x_1, x_2, x_3$ with the following four clauses:
	\[ c_1 := \lnot x_1 \vee \lnot x_2 \vee x_3 \qquad  c_2 := \lnot x_1 \vee  x_2 \vee x_3 \qquad  c_3 := x_1 \qquad c_4 := \lnot x_2 \vee \lnot x_3 \]
	An example of an interpretation is $\{\lnot x_2, x_3\}$. The \emph{maximal consequence} of formula $\formulac$ is $\{x_1, \lnot x_2, x_3\}$.
\end{example}

In the following, we introduce key properties on unsatisfiable subsets on which we build our algorithms.

\begin{definition}
A \emph{Minimal Unsatisfiable Subset} (MUS) of 
\F is an unsatisfiable subset $\m{S}$ of $\F$ for which every strict subset of $\m{S} $ is satisfiable. 
\muses{\F} denotes the set of MUSs of \F. 
\end{definition}

\begin{definition}
A set $\m{S} \subseteq \formula$ is a \emph{Maximal Satisfiable Subset} (MSS) of $ \formula$ if $\m{S}$ is satisfiable and for all $\m{S}'$ with $\m{S}  \subsetneq  \m{S}'\subseteq\formula $, $\m{S}'$ is unsatisfiable.
\end{definition}

\begin{definition}
A set $\m{S} \subseteq \formula$ is a \emph{correction subset} of \formula if $\formula\setminus\m{S}$ is satisfiable. 
Such a $\m{S}$ is a \emph{minimal correction subset} (MCS)  of \formula if no strict subset of $\m{S}$ is also a correction subset. 
\mcses{\F} denotes the set of minimal correction subsets of \F. 
\end{definition}
Each  MCS of \formula is the complement of an MSS of \formula and vice versa. 	

\begin{definition}\label{def:minimal-hs}
Given a collection of sets $\m{H}$, a hitting set of $\m{H}$ is a set $h$ such that  $h \cap C \neq \emptyset$ for every $C \in \m{H}$. A hitting set is \emph{minimal} if no strict subset of it is also a hitting set.

\end{definition}

The next proposition is the well-known hitting set duality \cite{DBLP:journals/jar/LiffitonS08,ai/Reiter87}  between MCSs and MUSs that forms the basis of our algorithms, as well as algorithms to compute MSSs \cite{DBLP:conf/sat/DaviesB13} and \emph{cardinality-minimal} MUSs \cite{ignatiev2015smallest}.

\begin{proposition}\label{prop:MCS-MUS-hittingset}
A set  $\m{S} \subseteq \formula$ is an MCS of $ \formula$ iff it is a \emph{minimal hitting set} of \muses{\formula}.
A set  $\m{S} \subseteq \formula$ is a MUS of $ \formula$ iff it is a \emph{minimal hitting set} of \mcses{\formula}.
\end{proposition}

Consider the following examples to illustrate the concepts related to unsatisfiability introduced in the previous definitions:
\begin{example}\label{example:unsat}
	Let $\formulac_2$ be the unsatisfiable CNF formula over atoms $x_1, x_2, x_3$ with the following clauses:
	\[ c_1 := \lnot x_1 \vee \lnot x_2 \vee x_3 \qquad  c_2 := \lnot x_1 \vee  x_2 \vee x_3  \qquad c_3 := \lnot x_2 \vee \lnot x_3 \qquad  c_4 := x_1\]
	\[ c_5 := \lnot x_1 \vee  x_2 \qquad  c_6 := \lnot x_1 \vee  \lnot x_3 \]

\end{example}
In \cref{example:unsat}, the subset of clauses $\{c_1, c_2, c_4, c_6\}$ is an example of a Minimal Unsatisfiable Subset (MUS) and $ \{\underline{c_1}\},\{\underline{c_4}\}, \{\underline{c_2},c_5\},\{c_5, \underline{c_6}\}, \{c_3, \underline{c_6}\}$ are examples of minimal correction subsets (MCSs). We observe that MUS $\{c_1, c_2, c_4, c_6\}$ hits at least one clause of each MCS.

\section{Motivation}\label{sec:motivation}

Our work is motivated by the problem of explaining the solution to constraint satisfaction problems, or how to explain the information entailed from it, through a sequence of simple explanation steps. This can be used to teach people problem solving skills; to search for mistakes in the problem formulation that lead to undesired models; to compare the difficulty of related satisfaction problems (through the number and complexity of steps required) and in systems that provide interactive tutoring.  
\subsection{Step-wise Explanation Generation}
Our original explanation generation algorithm \cite{ecai/BogaertsGCG20}, shown in \cref{alg:explanation-sequence}, starts from a formula $\formulac$ (in the application it comes from a high-level CSP), an \emph{initial interpretation} $I_0$ (here also seen as a conjunction of literals) and a cost function $f$ which quantifies the difficulty of an explanation step, e.g.~by means of a weight for each clause. 

An \emph{explanation step} is an implication $I' \wedge \formulac' \implies N$ where 
\begin{itemize}
	\item $I'$ is a subset of already derived literals I;
	\item $\formulac'$ is a subset of the constraints of the input formula $\formulac$; and
	\item $N$ is a set of literals entailed by $I'$ and $\formulac'$ which are not yet explained.
\end{itemize} 

The goal is to find a sequence of \emph{simple} explanation steps that explain the \emph{maximal consequence} \Iend, also referred to as the \emph{end interpretation}, entailed by the given initial interpretation $I_0$ (line~\ref{alg:explanation-sequence:prop-init-interpretation}).
In each explanation step, some literals from \Iend are explained, resulting in a precision-increasing sequence of interpretations $I_0 \preceq I_1 \preceq \dots I_{end}$.
Therefore, the intermediate interpretation $I$ at a given step in this sequence will consist of all the literals ``derived so far'' and the remaining literals to be explained will explain correspond to $\Iend \setminus I$.

\begin{algorithm}[!h]
	\DontPrintSemicolon
	\caption{$\mathtt{explain}(\formulac,f,I_0)$}
		\label{alg:explanation-sequence}
	$\Iend \gets \mathtt{propagate}(I \wedge \formulac) $\label{alg:explanation-sequence:prop-init-interpretation}\;
	Seq $\gets$ empty sequence\;
	$	I \gets I_0$\;
	\While{$I \neq Iend$}{
		$S \gets \onestep(\formulac, f, I , \Iend)$\label{alg:explanation-sequence:mus}\;
		$I' \gets I \cap S$ \label{alg:explanation-sequence:literals}\;
		$\formulac' \gets C \cap S $  \label{alg:explanation-sequence:constraints}\;
		$N \gets \mathtt{propagate}(I' \wedge \formulac')$ \label{alg:explanation-sequence:prop-expl}\;
		append $(I', \formulac', N)$ to Seq\;
		$I \gets I \cup N$\;
	}
	\Return{Seq} \;
	
\end{algorithm}

The subset returned by \onestep (line~\ref{alg:explanation-sequence:mus} of Algorithm~\ref{alg:explanation-sequence}) is then mapped back to the set of derived literals $I'$ (line~\ref{alg:explanation-sequence:literals}) and constraints $C'$ (line~\ref{alg:explanation-sequence:constraints}). 
Finally, at line~\ref{alg:explanation-sequence:prop-expl}, we propagate the used literals $I'$ and constraints $C'$ present in the found subset to derive new information $N'$ that holds at the intersection of all models of $I' \wedge C'$, i.e., we compute the maximal consequence of $I' \wedge C'$.

\begin{example}\label{example:intro2}
	Let $\formularunning$ 	be the CNF formula over atoms $x_1, x_2, x_3$ with the following four clauses:
	\[ c_1 := \lnot x_1 \vee \lnot x_2 \vee x_3 \qquad  c_2 := \lnot x_1 \vee  x_2 \vee x_3 \qquad  c_3 := x_1 \qquad c_4 := \lnot x_2 \vee \lnot x_3 \]
	
	Let $I = \{\lnot x_2\}$ be the given \emph{interpretation} with corresponding maximal consequence $\Iend = \{x_1, \lnot x_2,  x_3\}$.
	$\Iend \setminus I = \{x_1, x_3\}$ represents the set of remaining literals to explain and its negation is denoted by $\negset{\Iend \setminus I} = \{\lnot x_1, \lnot x_3\}$. 
	An example of an explanation step is \[\lnot x_2 \wedge c_2 \wedge c_3 \implies x_3 \] where we use derived literal $\lnot x_2$ and constraints $ c_2$ and $ c_3$ to infer $x_3$.
	
\end{example}

Note that all explanations in the examples are expressed in terms of this logical representation, however, the explanations can be translated back to the original input (e.g. natural language clues,, alldifferent constraints, ...).

\subsection{Explanation Step}

The key part of \onestep is the search for the best explanation step given an interpretation $I$ of the literals derived so far.
The procedure depicted in \cref{alg:oneStep} shows the gist of how this is done. 
It takes as input the problem constraints \formulac, a cost function $f$ quantifying the quality of explanations, an interpretation $I$ containing all literals derived so far in the sequence, and the interpretation to be explained $\Iend$. 

To compute an explanation step, this procedure iterates over the literals to be explained (line \ref{alg:onestep:repetition:mus-call} of \cref{alg:oneStep}), computes for each of them an associated MUS (line \ref{alg:onestep:mus-call}), and then selects the lowest cost one from the found MUSs (line \ref{alg:oneStep:ifcheck}).
The reason this works is a one-to-one correspondence between MUSs of $\formulac \land I \land \neg \ell$ and so-called \emph{non-redundant explanation} of $\ell$ in terms of (subsets of) $\formulac$ and $I$~\cite{ecai/BogaertsGCG20}.

\begin{algorithm}[!h]
	\DontPrintSemicolon
	\caption{$\onestep(\formulac,f,I,\Iend)$}
	\label{alg:oneStep}
	${\cal S}_{best} \gets \mathit{nil}$\;
	\For{$\ell \in \{\Iend \setminus I\}$\label{alg:onestep:repetition:mus-call}}{
		${\cal S} \gets \call{MUS}{(\formulac \land I \land \neg \ell)}$\label{alg:onestep:mus-call}\;
		\If{$f({\cal S})<f({\cal S}_{best})$\label{alg:oneStep:ifcheck}}{
			${\cal S}_{best} \gets {\cal S}$\;
		}
	}
	\Return{${\cal S}_{best}$} \;

\end{algorithm}
\begin{example}[\textit{\cref{example:intro2} continued}]\label{example:intro3}
	For computing a non-redundant explanation of $x_3$ zith given interpretation $I  = \{\lnot x_2\}$, we simply extract a \[MUS(\formulac \land I \land \neg \ell) = MUS(\formulac \wedge \{\lnot x_2\} \wedge \{\lnot x_3\}) = \{ c_2 , c_3, \{\lnot x_2\}, \{\lnot x_3\} \}\] also written as $\lnot x_2 \wedge c_2 \wedge c_3 \implies x_3$.
\end{example}

Throughout the paper, we consider cost functions $f$ that map a subset $\m{S} \subseteq \formulac \wedge \Iend \wedge \overline{\Iend \setminus I}$ to a numerical value $f(\m{S})$, namely, a weighted sum on the constraints present in the subset $\m{S}$.  $\overline{\Iend \setminus I} $ refers to the negation of the remaining literals to explain $\{\lnot \ell \mid \ell \in \Iend \setminus I\}$ introduced to compute non-redundant explanation steps with MUSs.

\subsection{Concluding Remarks}\label{sec:motivation:conluding-remarks}

Experiments have shown that such a MUS-based approach can easily take hours, especially when, at every explanation step, repeated MUS calls are performed for each remaining literal to explain (lines \ref{alg:onestep:repetition:mus-call}-\ref{alg:onestep:mus-call} of \cref{alg:oneStep}) to increase the chance of finding a low-cost MUS. 
Furthermore, MUS extraction algorithms do not guarantee of $\subseteq$-minimality or optimality with respect to a given cost function $f$. Therefore, in \citet{bogaerts2021framework}, we handle the absence of $\subseteq$-minimality and optimality guarantees by heuristically considering increasingly larger subsets of the unsatisfiable formula ($\formulac \land I \land \neg l$).

Hence, there is a need for algorithmic improvements to make it more practical.
We see three main points of improvement, all of which are addressed by our generic OCUS algorithm presented in the next section. 

\begin{itemize}
	
\item First, since the algorithm is based on \call{MUS} calls, there is no guarantee that the explanation found is indeed optimal (with respect to the given cost function). Performing multiple MUS calls is only a heuristic used to circumvent the restriction that \textit{there are no algorithms for cost-based unsatisfiable subset \textbf{optimization}}. 

\item Second, this algorithm uses \call{MUS} calls for every literal $ \ell \in \Iend \setminus I$ to explain separately. The goal of all these calls is to find a single unsatisfiable subset of $\formulac \land I \land \overline{(\Iend\setminus I)}$ that contains \emph{exactly one literal} from $\overline{(\Iend\setminus I)}= \{ \lnot \ell \mid \ell \in \Iend \setminus I\}$. This raises the question of whether it is possible \textit{to compute a single (optimal) unsatisfiable subset \textbf{subject to constraints}}, where in our case, the constraint is to include exactly one literal from $\overline{(\Iend\setminus I)}$. 

\item Third, the algorithm that computes an entire explanation sequence makes use of repeated calls to \onestep, and will therefore solve many similar problems. This raises the question of \textit{\textbf{incrementality}: can we reuse the computed data structures to achieve speedups in later calls?}

\end{itemize}

In the next section, we introduce the concept of \emph{Optimal Constrained Unsatisfiable Subsets} to address the first two main points of improvement. We address the third point in \cref{sec:efficient-ocus} along with other optimizations to speed up the generation of explanation \emph{sequence}.

\section{Optimal Constrained Unsatisfiable Subsets}\label{sec:OCUS}

	In this section, we address the first two challenges highlighted at the end of \cref{sec:motivation:conluding-remarks}, namely 
	\begin{inparaenum}[(1)]
		\item  the \emph{lack of optimality guarantees} when relying on MUS extraction methods (and heuristics) to compute an explanation step, and 
		\item whereas the optimal explanation of a single literal can be formalized as an \emph{optimal} MUS (with respect to a given objective), finding the optimal next explanation step over all literals remains an open question. 
	\end{inparaenum}
	 To tackle these, we introduce the concept of an \emph{Optimal Constrained Unsatisfiable Subset (OCUS)} and propose an algorithm for computing one.

\begin{definition}\label{definition:ocus}
Let $\formula$ be a formula, $f:2^{\formula} \to \nat$ a cost function and  $p$ a predicate $p: 2^{\formula} \to \{\mathit{true}, \mathit{false} \}$. 
We call $\m{S} \subseteq \formula$ an \emph{Optimal Constrained Unsatisfiable Subset (OCUS)} of \formula (with respect to $f$ and $p$) if 
\begin{itemize}                                      
\item $\m{S}$ is unsatisfiable,
\item $p(\m{S})$ is true
\item all other unsatisfiable $\m{S}'\subseteq \formula$ for which $p(\m{S}')$ is true satisfy $f(\m{S}')\geq f(\m{S})$.
\end{itemize}
\end{definition}

\begin{proposition}
	Let \emph{\formula} be a CNF formula, ${p}$ be a predicate specified as a CNF over (meta)-variables indicating the inclusion of clauses of \emph{\formula}, and ${f}$ be a cost-function obtained by assigning a weight to each such meta-variable, then the problem complexity of finding an OCUS is \emph{$\Sigma^\emph{P}_2$-complete}.
\end{proposition}
\begin{proof}
		If we assume that the predicate $p$ is specified itself as a CNF over (meta)-variables indicating the inclusion of clauses of $\m{F}$, and $f$ is obtained by assigning a weight to each such meta-variable, then the complexity of the problem of finding an OCUS is the same as that of the SMUS (cardinality-minimal or `Smallest' MUS) problem  \cite{ignatiev2015smallest}: the associated decision problem is $\Sigma^P_2$-complete.
		Hardness follows from the fact that SMUS is a special case of OCUS; containment follows - intuitively - from the fact that this can be encoded as an $\exists\forall$-QBF using a Boolean circuit encoding of the costs. 
\end{proof}

In the next sections, we first explain how OCUS can be used to compute an explanation step and then propose an algorithm for computing OCUSs using the well-known hitting set duality between MUSs and MCSs \cite{DBLP:journals/jar/LiffitonS08,ai/Reiter87}.

\subsection{\onestep with OCUS}

\onestep implicitly uses an ``exactly one of'' constraint on the set of literals to explain $ \{\ell \in \Iend \setminus I\}$, and the way this is done is by considering each literal $\ell$ from that set separately. It searches for a good, but not necessarily \emph{optimal} CUS (an \underline{U}nsatisfiable \underline{S}ubset that satisfies the \underline{C}onstraint in question). However, the goal is to find an optimal one.
Therefore, when considering the procedure \onestep from the perspective of OCUS defined above, the task of the procedure is to compute an OCUS of the formula $\formula := \formulac\land I\land \overline{\Iend\setminus I}$ where $p$ is the predicate that holds for subsets containing \emph{exactly one} literal of $\overline{\Iend\setminus I}$. 
The pseudocode for this is shown in \cref{alg:oneStepOCUS}.

\begin{algorithm}[!h]
	\DontPrintSemicolon
	
	\caption{$\onestepo(\formulac,f,I,\Iend)$}
	\label{alg:oneStepOCUS}
$\formula \gets \formulac\land I\land \overline{\Iend\setminus I}$\;
$p \gets (S\mapsto \# (S\cap \overline{\Iend\setminus  I})=1)$\label{alg:oneStepOCUS:exactlyone}\;
$\m{S} , \mathit{status} \gets \call{OCUS}(\formula, f, p)$\label{alg:oneStepOCUS:ocus-call}\;
\lIf{$\mathit{status} \neq \mathit{FAILURE}$}{\Return{$\m{S}$} }
\lElse{\Return{$\emptyset$}}
\end{algorithm}

\paragraph{Explaining 1 Literal per Step.} We define the `exactly one' constraint $p$ as the equality constraint shown on line \ref{alg:oneStepOCUS:exactlyone} of \cref{alg:oneStepOCUS}. $p$ evaluates to \textit{true} if the size of the intersection of the negated literals to explain $\Iend\setminus I$ with the subset $\m{S}$ of $\formula$ considered is equal to 1, otherwise $p$ evaluates to \textit{false}. 

The call to \comus on line \ref{alg:oneStepOCUS:ocus-call} of \onestepo will never lead to the $\mathit{FAILURE}$ case, since an OCUS is guaranteed to exist for the input cost function $f$, predicate $p$ and formula $\formula$ we are considering. Failure may occur for a different predicate $p$. For example, consider the predicate $p: (\m{S} \mapsto f(S) \leq v)$ of \cref{alg:oneStepBoundBart}, which evaluates to true if a subset $\m{S}$ has a value less than a given bound $v$. In this case it can lead to failure if no OCUS exists for a given bound $v$.  

In the next section, we provide the details on how to compute an OCUS on line~\ref{alg:oneStepOCUS:ocus-call}.

\subsection{Computing an OCUS}\label{sec:computing_an_ocus}

In order to compute an OCUS of a given formula, we propose to build on the hitting set duality of \cref{prop:MCS-MUS-hittingset}. 
For this, we will assume to have access to a solver \cohs that can compute hitting sets of a given collection of sets that are \emph{optimal} (w.r.t.~a given cost function $f$) among all hitting sets \emph{satisfying a condition $p$}. 
The choice of the underlying hitting set solver will thus determine which types of cost functions and constraints are possible. 

In our implementation, we use a cost function $f$ which is encoded as a linear term (weighted sum), where e.g. constraints are given a larger weight than already derived literals. 
For example, (unit) clauses representing previously derived facts can be given small weights, and regular clauses can be given large weights, so that explanations are penalized for including clauses when previously derived facts can be used instead.
Condition $p$ can be easily encoded as a linear constraint (see \cref{eq:exactly-one-constraint} for an example), 
thus allowing the use of highly optimized mixed integer programming (MIP) solvers to compute optimal hitting sets. 
In the following, we explain how the conditional optimal hitting set problem \cohs can be encoded into MIP to reason over combinations of clauses and literals (hitting sets) of the unsatisfiable formula.

\paragraph{Hitting Set Problem} Given $\formula$, we define a MIP decision variable $d_c$ for every clause $c\in \formula$, and write  $\decisionvariables = \{d_c \mid c \in \formula\}$ for the set of all such variables. We assume a given collection of sets-to-hit $\setstohit$.
The goal is to find a hitting set $\hittingset \subseteq D$ that hits every set-to-hit at least once (\cref{eq:set-to-hit}), satisfies predicate $p$ (\cref{eq:constraint}), and minimizes $f(D)$ (\cref{eq:minimization}).
The \cohs formulation is as follows:
\begin{align}
\small
\mini_{{\hittingset \subseteq \decisionvariables}} \quad & f(\hittingset) \label{eq:minimization}\\ 
s.t. \quad & p(\hittingset) \label{eq:constraint}\\
  & \sum_{c\in \settohit}\decisionvariable_c \geq 1, \quad & \forall \settohit \in \setstohit \label{eq:set-to-hit}\\
  & \decisionvariable_c \in \{0,1\}, \quad  & \forall \decisionvariable_c \in \decisionvariables     \label{eq:decision-variables}
\end{align}

In the case of \comus, every set-to-hit corresponds to a correction subset of $\formula$. 

\paragraph{OCUS Algorithm} Our generic algorithm for computing OCUSs is depicted in \cref{alg:comus}. It combines the hitting set-based approach for MUSs of \citet{ignatiev2015smallest} with the use of a MIP solver for (weighted) hitting sets as proposed for maximum satisfiability by \citet{DBLP:conf/sat/DaviesB13}. 
The key novelty is the ability to add structural constraints to the hitting set solver, without impacting the duality principles of \cref{prop:MCS-MUS-hittingset}, as we will show.

\begin{algorithm}[!h]
	\DontPrintSemicolon
	$\setstohit  \gets \emptyset$ \; 
	\While{true}{
		$\m{S}, \mathit{status} \gets \cohs(\setstohit,f,p) $ \label{line:opt} \;
		\lIf{$\mathit{status} = \mathit{FAILURE}$}{\Return{$(\emptyset, \mathit{status})$}}

		\If{ $\lnot \sat(\m{S})$ \label{alg:ocus-sat-check}}{
			\Return{$(\m{S}, \mathit{status})$} \;
		}
		$\m{K} \gets \corrsubsets(\subsetT, \formula)$ \label{line:grow}\;
		$\setstohit  \gets \setstohit  \cup \m{K} $  \label{alg:ocus:complement}\;
	}
	\caption{$\comus(\formula,f,p)$ }
	\label{alg:comus}
\end{algorithm}

The algorithm alternates calls to a hitting set solver with calls to a \sat oracle on a subset $\m{S}$ of $\formula$. 
In case the \sat oracle returns true, i.e., subset $\m{S}$ is satisfiable, $\m{K}$ a set of subsets of $F\setminus S$ is returned by \corrsubsets and added to the collection of sets-to-hit \setstohit. 

\paragraph{Smallest MUS of \citet{ignatiev2015smallest}}
The key differences with the SMUS algorithm are the calls to a \cohs solver (resp.~\texttt{MinimumHS}) and the \corrsubsets (resp.~\texttt{grow}) procedure.
In the SMUS algorithm, the purpose of \texttt{grow} is to expand a satisfiable subset $\m{S}$ of $\m{F}$ further such that its complement, a correction subset, is as small as possible.
Shrinking the correction subset as a result of \texttt{grow} finds stronger constraints on the sets-to-hit, since it restricts the choice of clauses to be selected. 

In our algorithm we make use of the \corrsubsets procedure which returns a non-empty \emph{set} of correction subsets. 
In the case of \comus, the calls for hitting sets will also take into account the cost ($f$), and the meta-level constraints ($p$).
As such, It is not clear a priori which properties a good \corrsubsets function should have here.
In \cref{sec:computing-correction-subsets}, we propose different domain-specific methods for enumerating multiple correction subsets. 

For the \emph{correctness} of the algorithm, all we need to know is that for a given satisfiable subset $\m{S}$, \corrsubsets returns a \emph{non-empty} set of correction subsets~$\m{K}$, where $\forall \ \m{C} \in \m{K} : \m{C} \subseteq (F \setminus S)$.
At any given step, if $\m{S}$ is satisfiable, $\setstohit$ is guaranteed to grow, since a non-empty set of correction subsets $\m{K}$ is returned that is disjoint from $\m{S}$. Therefore $\m{K}$ cannot be present in $\setstohit$. 
The \emph{completeness} and \emph{soundness} of the algorithm follow from the fact that the algorithm is guaranteed to terminate since there is a countable number of correction subsets of $\formula$, and from \cref{thm:soundcomplete}, which states that what is returned is indeed a solution and that a solution will be found if it exists.  

\begin{theorem}\label{thm:soundcomplete}
	Let $\m{H}$ be a set of correction subsets of \formula. 
	If $\m{S}$ is a hitting set of $\m{H}$ that is $f$-optimal among the hitting sets of $\m{H}$ that satisfy a predicate $p$, and  $\m{S}$ is unsatisfiable, then $\m{S}$ is an OCUS of $\formula$. 
	If  $\m{H}$ has no hitting sets satisfying $p$, then $\formula$ has no OCUSs.
\end{theorem}
\begin{proof}
For the first claim, it is clear that $\m{S}$ is unsatisfiable and satisfies $p$. Hence all we need to show is the $f$-optimality of $\m{S}$.
If there would exist some other unsatisfiable subset $\m{S}'$ that satisfies $p$ with $f(\m{S}')\leq f(\m{S})$, we know that $\m{S}'$ would hit every minimal correction set of $\m{F}$, and hence also every set in $\m{H}$ (since every correction set is the superset of a minimal correction set).
Since $\m{S}$ is $f$-optimal among hitting sets of $\m{H}$ that satisfy $p$, and since $\m{S}'$ also hits $\m{H}$ and satisfies $p$, it must be that $f(\m{S})=f(\m{S}')$. 

The second claim follows immediately from \cref{prop:MCS-MUS-hittingset} and the fact that an OCUS is an unsatisfiable subset of $\formula$. 
\end{proof}

Perhaps surprisingly, the correctness of the proposed algorithm does \emph{not} depend on the monotonicity properties of $f$ nor $p$. In principle, any (computable) cost function and condition on the unsatisfiable subsets can be used. In practice, however, one is bound by the limitations of the chosen hitting set solver.

As an illustration, we provide an example of one call to \onestepo (\cref{alg:oneStepOCUS}) and the corresponding \comus-call in detail (\cref{alg:comus}) for our running example:
\begin{example}[\textit{continued}]\label{ex:onestep-ocus}
	Recall the 4 previously introduced clauses from our running example $\formularunning := c_1 \wedge c_2 \wedge c_3 \wedge c_4 $ with:
			\[ c_1 := \lnot x_1 \vee \lnot x_2 \vee x_3 \qquad  c_2 := \lnot x_1 \vee  x_2 \vee x_3 \qquad  c_3 := x_1 \qquad c_4 := \lnot x_2 \vee \lnot x_3 \]
Consider a call to \onestepo with $ I = \emptyset$ and $\Iend = \{x_1, \lnot x_2, x_3\}$. We add the following three new clauses, which represent the complement of the literals to be derived $\overline{\Iend\setminus I}$:
\[ \{\lnot x_1\}\qquad  \{x_2\} \qquad\{\lnot x_3\}\]
The cost function $f$ is defined as a linear sum  
\begin{equation}
	f(\m{S}) = \sum_{c \in \m{S}} w_i\cdot c_i
\end{equation} over the following clause weights:
\begin{align*}
	\text{\textbf{\emph{Clause weights:}}} & \ w_{c_1} = 60 \  \qquad w_{c_2}=60  \qquad  \ \ \ w_{c_3}=100 \qquad w_{c_4}=100 \\
	\text{\textbf{\emph{$I \wedge \overline{\Iend \setminus I}$ weights:}}} & \  w_{\lnot x_1}=1  \qquad w_{ x_2}=1 \ \ \ \qquad w_{\lnot x_3}=1 
\end{align*}
We encode the same cost function in the MIP encoding as a weighted sum using the corresponding decision variables as follows:
\begin{equation}
	f(\m{S}) = \sum_{c \in \m{S}} w_c\cdot d_{c} 
\end{equation}

		To ensure that we only explain one literal at a time, we add predicate $p$ as
		\begin{equation}\label{eq:cost-function}
		p(S) := \# (S\cap \{\lnot x_1, x_2, \lnot x_3\})=1
		\end{equation}
	The MIP encoding of the predicate $p$ corresponds to
		\begin{equation}
		p(h):= \sum_{c \in \overline{\Iend \setminus I}} d_c = \decisionvariable_{\{\lnot x_1\}} + \decisionvariable_{\{x_2\}} + \decisionvariable_{\{\lnot x_3\}} = 1
		\label{eq:exactly-one-constraint}
	\end{equation}		

At \cref{alg:oneStepOCUS:ocus-call} of \onestepo, $\m{F}$ is constructed, consisting of :

\begin{equation}
\formula = \formularunning \land I \land \overline{(\Iend\setminus I)} =  c_1 \wedge \dots \wedge c_{4} \wedge \lnot x_1 \wedge  x_2 \wedge  \lnot x_3	
\label{eq:explain-one-step-ocus}
\end{equation}

	Finally, to generate an explanation step, we call \comus on formula $\formula$ with the given cost function $f$ and \emph{exactly-one} constraint $p$:
	\begin{equation}\comus( \formula, f, p)\end{equation}

	In this small example, the $\corrsubsets(\m{S}, \formula)$ procedure simply returns $\{F\setminus S\}$. 

\begin{table}[!h]
	\centering
	\begin{adjustbox}{max width=\textwidth}
		\begin{tabular}{lccc} 
			Step & $\m{S}$                             & $\sat(\m{S})$ & $\setstohit  \gets \setstohit  \cup  \corrsubsets(\m{S}, \formula) $ \\ 
			\toprule[2pt]
			&        &    &      $\emptyset$                                \\
			\midrule
			1. & $\{ \lnot x_3\}$   &$\mathit{true}$ &  $\{\{c_1, c_2, c_3, c_4, \lnot x_1, x_2 \}\}$\\\midrule
			2. & $\{ \lnot x_1\}$   &$\mathit{true}$ &   $\{  \dots, \{ c_1, c_2, c_3, c_4, x_2, \lnot x_3 \}\}$\\
			\midrule
			3. & $\{ x_2\}$   		&$\mathit{true}$&  $\{ \dots, \{c_1, c_2, c_3, c_4, \lnot x_1, \lnot x_3 \}\}$\\
			\midrule
			4. & $\{ c_1,  x_2\}$   	&$\mathit{true}$ &$\{ \dots, \{c_2, c_3, c_4, \lnot x_1, \lnot x_3 \}\}$\\
			\midrule
			5. & $\{c_2, x_2\}$ 	&$\mathit{true}$ &$\{ \dots, \{c_1, c_3, c_4, \lnot x_1, \lnot x_3 \}\}$\\
			\midrule
			6. & $\{c_1, \lnot x_3\}$ 	&$\mathit{true}$ &$\{ \dots, \{c_2, c_3, c_4, \lnot x_1, x_2 \}\}$\\
			\midrule
			7. & $\{c_2, \lnot x_1\}$ 	&$\mathit{true}$ &$\{ \dots, \{c_1, c_3, c_4,x_2,  \lnot x_3 \}\}$\\
			\midrule
			8. & $\{c_1, \lnot x_1\}$ 	&$\mathit{true}$ &$\{ \dots, \{c_2, c_3, c_4, x_2, \lnot x_3 \}\}$\\
			\midrule
			9. & $\{c_2, \lnot x_3 \}$ 	&$\mathit{true}$ &$\{ \dots, \{c_1, c_3, c_4, \lnot x_1, x_2 \}\}$\\
			\midrule
			10. & $\{ c_4, x_2\}$ 	&$\mathit{true}$ &$\{ \dots, \{c_1, c_2, c_3, \lnot x_1, \lnot x_3 \}\}$\\
			\midrule
			11. & $\{ c_3, x_2\}$ 	&$\mathit{true}$ &$\{ \dots, \{c_1, c_2, c_4, \lnot x_1, \lnot x_3 \}\}$\\
			\midrule
			12. & $\mathbf{\{ c_3, \lnot x_1\}}$ 	&$\mathit{false}$  & \\
			
	\end{tabular}\end{adjustbox}
	\caption{
		\cref{ex:onestep-ocus} - Intermediate steps when computing an \comus for \onestepo where $\corrsubsets(\m{S}, \formula) = \{\formula \setminus \m{S}\}$.	}
	\label{tab:example-explanation-generation-not-incremental}
\end{table}

	\cref{tab:example-explanation-generation-not-incremental} breaks down the intermediate steps of algorithm \ref{alg:comus} for generating an OCUS of given $\formula$, $f$ and $p$. 
	First, the collection of sets-to-hit \setstohit is initialized as the empty set. At each iteration, the hitting set solver searches for a cost-minimal assignment that hits all sets in \setstohit and that contains exactly one of $\{\lnot x_1, x_2, \lnot x_3\}$ (due to $p$). If the hitting set is unsatisfiable, it is guaranteed to be an OCUS.
	
 The first iterations show that the \comus algorithm first hits all the literals of $\overline{\Iend\setminus I}$ (steps 1-3 of \cref{tab:example-explanation-generation-not-incremental}) and then starts combining one literal of $\overline{\Iend\setminus I}$ 
  with the remaining clauses until an OCUS is found (step 12). 
  Finally, step 12 of \cref{tab:example-explanation-generation-not-incremental} signifies that using clause $c_3$ we can derive $x_1$, and more formally $c_3 \implies x_1$.

\end{example}

	\paragraph{Incremental MIP Solver} The \comus algorithm requires repeatedly computing hitting sets over an increasing collection of sets-to-hit. Initializing the MIP solver once and keeping it warm throughout the \comus iterations allows it to reuse information from previous solver calls to solve the current hitting set problem. Similar to \citet{DBLP:conf/sat/DaviesB13}, we notice a speed-up between 3 to 5 times by keeping the solver warm\footnote{We use the same setup as in the experiment section: a single core on a 10-core INTEL Xeon Gold 61482 (Skylake), a memory-limit of 8GB. The code is written on top of PySAT 0.1.7.dev1 \cite{pysat}, for MIP calls, we used Gurobi 9.1.2, and for the SAT calls MiniSat 2.2.} compared to initializing a new MIP solver instance at every iteration.

	As can be seen in \cref{tab:example-explanation-generation-not-incremental}, the \comus algorithm requires many intermediate steps to find an OCUS. 
	Note, for example, that the computed correction subsets contain more than 1 literal to explain that are not relevant for explaining the literal in the current hitting set. 
	Taking step 2 as an example, if $\{\lnot x_1\}$ is a hitting set, its corresponding correction subset $ \{ c_1, c_2, c_3, c_4, x_2, \lnot x_3 \}$ contains $\{x_2, \lnot x_3\}$ which cannot be taken. 

	Even though our running example has a rather small number of clauses, \comus needs to combine an increasingly large number of literals and clauses to find an OCUS. 
	Next, we investigate how to efficiently grow a given satisfiable subset in order to reduce the size of the corresponding correction subset. By imposing stronger restrictions on the hitting sets, we will be able to reduce the number of sets to hit.
	
	\subsection{Computing Correction Subsets}\label{sec:computing-correction-subsets}
	
	The \corrsubsets procedure generates a set of correction subsets starting from a given satisfiable subset $\subsetT$ of an unsatisfiable formula $\formula$. 
	However, calling \onestepo on our running example has shown that a naive `No grow':
	\begin{equation}\label{corrsubsets-no-grow}
		\corrsubsets(\subsetT, \formula) \gets \{\formula \setminus  \subsetT \}
	\end{equation}
	drastically increases the number of sets-to-hit required compared to using the model provided by the \sat solver. 
	This last observation suggests that the satisfiable subset $\m{S}$ should be \emph{efficiently grown} into a \emph{larger satisfiable subset} before computing the complement:
	\begin{equation}\label{corrsubsets-with-grow}
		\corrsubsets(\subsetT, \formula) \gets \{\formula \setminus  \grow(\subsetT, \formula) \}
	\end{equation}
	
	\subsubsection{Growing satisfiable subsets using domain-specific information} \label{para:domainspecificgrow}
	
	The goal of the \grow phase of the \corrsubsets procedure (see \cref{corrsubsets-with-grow}) is to turn $\m{S}$ into a larger satisfiable subformula of $\formula$. 
	The effect of this is that the complement added to \setstohit will be smaller, and hence imposes stronger restrictions on the hitting sets.
	
	There are multiple conflicting criteria that determine what makes an effective `grow' procedure.
	On the one hand, we want our subformula to be as large as possible (which would ultimately correspond to computing a maximal satisfiable subformula), 
	but on the other hand, we also want the procedure to be very efficient, as it is called in every iteration. 
	
	In the case of explanations, we make the following observations: 
	\begin{itemize}
		\item Our formula at hand (using the notation from the \onestepo algorithm) consists of three types of clauses: 
		\begin{inparaenum}[(1)]
			\item (translations of) the problem constraints (this is \formulac) 
			\item literals representing the assignment found (this is $I$), and 
			\item the negations of literals not yet derived (this is $\overline{\Iend\setminus I}$). 
		\end{inparaenum}
		\item $\formulac$ and $I$ together are satisfiable, with assignment $I_{end}$, and \emph{mutually supportive}, by this we mean that making more clauses in \formulac true, more literals in $I$ will automatically become true and vice versa. 
		\item The constraint $p$ enforces that each hitting set will contain \textbf{exactly} one literal of  $\overline{\Iend\setminus I}$
	\end{itemize}
	
	Since the restrictions on the third type of elements of $\formula$ are already strong, it makes sense 
	to search for a \emph{maximal} satisfiable subset of $\formulac\cup I$ with hard constraints that $\m{S}$ should be satisfied, using a call to an efficient (partial) \maxsat~solver. 
	
	Furthermore, we can initialize this call as well as any call to a \sat solver with the polarities for all variables set to the value they take in $\Iend$. 
	
	We evaluate different grow strategies as part of the $\corrsubsets(\m{S}, \formula)$ procedure in the experiments section including:
	\begin{description}
		\item[\sat] extracts a satisfying model from the \sat solver to turn $\m{S}$ into a larger satisfiable subset. This grow will be considered the baseline for comparing the other grow-variants.
		\item[\subsetmaxsat] extends the satisfiable subset computed by \sat by looping over every remaining clause $c \in \formula \setminus \m{S}$.  If $\m{S} \cup \{c\}$ is satisfiable, then clause $c$ is added to $\m{S}$ as well as any other clause from $c' \in \formula \setminus \m{S} \cup \{c\}$ that is satisfied in the model found by the SAT solver. 
		\item[\domspecmaxsat] grows a satisfiable subset $\m{S}$ with a \maxsat~solver using only the previously derived facts and the original constraints. 
		\item[\maxsatfull] grows satisfiable subset $\m{S}$ with a \maxsat~solver using the full unsatisfiable formula $\mathcal{F}$.
	\end{description} 
	In the experiments, we omit the `No grow' procedure where the complement $\{\formula \setminus \m{S}\}$ is returned by \corrsubsets. Not growing produces large sets-to-hit as seen in \cref{tab:example-explanation-generation-not-incremental} of the running example.
	Additional experiments comparing the effectiveness of the naive `No grow' to growing with the \sat solver procedure shows that `No grow' leads to significantly longer \comus runtimes.

	Finally, \cref{tab:example-explanation-generation-not-incremental} showed that \comus has to combine an increasingly large number of literals and clauses. In the next section, we analyse whether we can break the more general OCUS problem into smaller subproblems, similar to \cref{alg:oneStep}, where instead of searching for a MUS, we search for an \emph{Optimal Unsatisfiable Subset} (OUS) and select the best one.

\section{Multiple Optimal Unsatisfiable Subsets}\label{sec:OUS}

Preliminary experiments have shown that most of the time ($\sim90\%$ of the time) is spent searching for hitting sets when generating an explanation step with \comus. 
The main reason for this is that the hitting set solver needs to consider an increasingly large collection of sets-to-hit, potentially searching over an exponential number of literals and clauses (see Table~\ref{tab:example-explanation-generation-not-incremental} of \cref{ex:onestep-ocus}).

In this section, we first analyze if instead of working OCUS-based, we can split up the \comus-call into individual calls that compute Optimal Unsatisfiable Subsets (OUSs) for every literal by replacing a MUS call to OUS in \cref{alg:oneStep}.

\subsection{Bounded OCUS}

Since OUS, without an additional predicate, is a special case of OCUS, we can take advantage of \cref{prop:MCS-MUS-hittingset} and reuse the \comus algorithm with a trivially true $p$, i.e., 
$\omus{}(\formula,f):=\comus{}(\formula,f,\ltrue) $ for each $\formula$ and~$f$. 
However, the switch from \texttt{MUS} to \omus{} in \cref{alg:oneStep} still requires looping over every literal and computing the OUS, potentially introducing overhead compared to the single \comus call.
However, we can use the OUS obtained in one iteration, to infer a bound on the score that must be achieved in subsequent OUS calls.

\paragraph{Upper Bound} Every MUS or OUS computed at \cref{alg:onestep:mus-call} of \cref{alg:oneStep} provides an upper bound on the cost, which should be improved in the next iteration.
By keeping track of the best candidate explanation, its corresponding cost can be considered the current best upper bound on the cost of the OUSs of the remaining literals to explain.

\paragraph{Lower Bound} Every hitting set computed inside the \omus{} algorithm produces a lower bound on the best cost that can be obtained, even the satisfying ones. Indeed, the candidate hitting set returned on line~\ref{line:opt} of \cref{alg:comus} is guaranteed to be the lowest-cost one. Consequently, the cost of the best candidate explanation so far can be used as an early stopping criterion: if the cost of the current hitting set is larger than the cost of the best explanation so far, \omus{} will not be able to find a better (cheaper) unsatisfiable subset $\m{S}$ for that literal.

In fact, such a \emph{bounded OCUS} call is naturally obtained by doing an \comus call with as constraint $p(S):= f(\m{S}) \leq f(\m{S}_{best})$.

\paragraph{Bounded OCUS-based Explanations} \onestepomusbounded{} in \cref{alg:oneStepBoundBart} uses calls to the \comus algorithm for every individual literal to compute the next best explanation step. The algorithm keeps track of the current best OCUS candidate $\m{S}_{best}$.
This $\m{S}_{best}$ is only updated if the \comus algorithm is able to find an OCUS that is cheaper than the current upper bound $f(\m{S}_{best})$. 
Predicate $f(\m{S}) \leq f(\m{S}_{best})$ of the \comus-call at line \ref{alg:oneStepBound:cost_check} of \cref{alg:oneStepBoundBart} allows us to ensure that the cost of the hitting set does not exceed the upper bound. In case it does happen, the hitting set solver will return a failure message meaning that a better candidate explanation cannot be computed for that literal given the current interpretation $I$.

\begin{algorithm}[!ht]
	\DontPrintSemicolon
	\caption{$\onestepomusbounded(\formulac,f,I,\Iend)$}
	\label{alg:oneStepBoundBart}
	$\m{S}_{best} \gets \mathit{nil}$ \;
	${\cal S}^\ell_{best} \gets \mathit{nil}$ for each $\ell$ \quad \quad (or from previous iteration) \;
	\For{$\ell \in \{\Iend \setminus I\}$ sorted by $f({\cal S}^\ell_{best})$\label{alg:oneStepBoundBart:sorting}}{
		${\cal S}^\ell_{best}, \mathit{status} \gets \call{OCUS}{( (\formulac \land I \land \neg l), f, f(\m{S}) \leq f(\m{S}_{best}) )}$\label{alg:oneStepBound:cost_check}\;
		\If{$\mathit{status} \neq \mathit{FAILURE}$ }{
			$\m{S}_{best} \gets \m{S}^\ell_{best}$\;
		}
	}
	\Return{${\cal S}_{best}$} 
\end{algorithm}

\paragraph{Literal Sorting} Obtaining a good upper-bound quickly can further reduce runtime. 
We can heuristically aid this by keeping track of $S^\ell_{best}$ across explanation steps, and then use its score $f(S^\ell_{best})$ to sort the literals at line~\ref{alg:oneStepBoundBart:sorting}. The literal sorting ensures we first try the cheapest $S^\ell_{best}$ from a previous explanation step since these are more likely to provide a good upper bound on the cost of the next candidate OCUS.

\subsection{Interleaving OCUS Calls for Different Literals: a Special-case OCUS Algorithm} 
The case to avoid for \onestepomusbounded{} is that an \comus call for a literal takes many hitting set iterations, and returns an `expensive' OCUS with lower-cost OCUSs to be found for other literals. 

Conceptually, one should only do hitting set iterations for the most promising literal, one that is most likely to produce an OCUS with the lowest cost. 
Indeed, this is what the original \comus algorithm with an `exactly one of' constraint is built for: to choose freely among all possible hitting sets across the different literals in order to find the globally optimal next candidate OUS. 

For this special case, where the constraint $p$ is that exactly one of a set of literals must be chosen, we can manually decompose the problem to iteratively search for the best hitting set across the independent problems.
In such an approach, we do not repeatedly call (bounded) \comus until optimality, but do one hitting-set iteration at a time; each time continuing with one hitting-set iteration of the most promising literal. 

This is shown in \cref{alg:oneStepIter}. Every literal to explain $\ell$ is associated with:~\begin{inparaenum}[(1)]
	\item its current \emph{collection of sets to hit} ${\setstohit_{\ell}}$; and
	\item corresponding \emph{optimal hitting set} ${\m{S}_\ell}$ (initially the empty set for both) as well as
	\item it's corresponding \emph{cost} ${f(\m{S}_\ell)}$.
\end{inparaenum}
These are stored in a priority queue, sorted by the cost.

\begin{algorithm}[!ht]
	\DontPrintSemicolon
	\caption{$\iterativeonestep(\formulac,f,I,\Iend)$}
	\label{alg:oneStepIter}
	$queue \gets$ \texttt{InitializePriorityQueue($(\ell, \emptyset,\emptyset): 0 \mid \forall \ell \in \Iend \setminus I$)}\; 
	\While{$(\ell, \m{S}_\ell, \setstohit_{\ell}) \gets queue.pop()$}{
			\If{ $\lnot \sat(\m{S}_\ell)$}{\label{alg:onestepIter-sat-check}
				\Return{$\m{S}_\ell$} \;
			}
		
			$\m{K} \gets \corrsubsets(\subsetT_\ell, (\formulac \land I \land \neg \ell))$\;
		$\setstohit_\ell  \gets \setstohit_\ell  \cup \m{K} $\;
		$\m{S}_\ell \gets OptHittingSet(\setstohit_\ell,f) $  \;
			\texttt{queue.push($(\ell, \m{S}_\ell, \setstohit_{\ell}): f(\m{S_\ell})$)}\;
		}
	\end{algorithm}

\iterativeonestep repeatedly extracts the best literal-to-explain and corresponding hitting set out of the priority queue. Similar to Algorithm~\ref{alg:comus}, if the corresponding hitting set is unsatisfiable, it is guaranteed to be the cost-minimal OUS. 
This is because the queue ensures that this hitting set is the lowest scoring hitting set across all literals and because each hitting set is guaranteed to be an optimal hitting set of its collected sets-to-hit $\setstohit_\ell$. 
If, on the other hand, the hitting set is satisfiable, a number of correction subsets are extracted from the literal-specific unsatisfiable formula and added to its respective collection of sets to hit. 
Finally, a new hitting set is computed and this information is pushed back into the priority queue. 

	The intermediate steps of \comus depicted in Table~\ref{tab:example-explanation-generation-not-incremental} of \cref{ex:onestep-ocus} shows that \comus needs to consider \emph{many combinations of clauses and literals} for \emph{all literals to explain}. 
	Whereas \comussplit reasons over a smaller unsatisfiable formula containing literals relevant for the literal to explain, and only expands the most promising literal.
	In the experiments, we compare which \onestep-* configuration  (\comus, \comusbound, or \comussplit) is the fastest for computing explanations.

 	In the following section, we analyse how to exploit the fact that OCUS and its variants have to be called repeatedly on an unsatisfiable formula that is incrementally extended when generating a sequence of explanations.
	We then consider how to reduce the number and size of sets-to-hit, e.g. to include only information relevant to each literal-to-explain, and how to generate small ($p$-)disjoint correction subsets from a given hitting set.
 	
\section{Efficiently Computing Optimal Explanations}\label{sec:efficient-ocus}

	Up until now, we have investigated how to speed-up the generation of an explanation step from the perspective of OCUS as an oracle.
	In the following, we discuss optimizations applicable to the \texttt{O(C)US} algorithms that are specific to explanation sequence generation, though they can also be used when other forms of domain knowledge are present. 

\subsection{Incremental OCUS Computation}\label{sec:ocusEx}
Inherently, generating a sequence of explanations still requires many O(C)US calls. 
Indeed, a greedy sequence construction algorithm 
calls an \onestep variant repeatedly with a growing interpretation $I$ until $I=\Iend$.
All of these calls to \onestep, and hence O(C)US, are done with very similar input (the set of constraints does not change, and the $I$ slowly grows between two calls). For this reason, it makes sense that information computed during one of the earlier stages can be useful in later stages as well. 
The main question is: 
\begin{quote}
	Suppose two \comus calls are done, first with inputs $\formula_1$, $f_1$, and $p_1$, and later with $\formula_2$, $f_2$, and $p_2$; \emph{how can we make use as much as possible of the data computations of the first call to speed-up the second call?}
\end{quote}
 The answer is surprisingly elegant. The most important data \comus keeps track of is the collection \setstohit of correction subsets that need to be hit.

\subsubsection{Bootstrapping $\setstohit$ with Satisfiable Subsets}\label{sec:bootstrapping-sat-subsets} This collection in itself is not useful for transfer between two calls, since -- unless we assume that $\formula_2$ is a subset of $\formula_1$ -- there is no reason to assume that a set in $\setstohit_1$ is also a correction subset of $\formula_2$ in the second call. 
However, each set $H$ in $\setstohit$ is the complement (with respect to the formula at hand) of a \emph{satisfiable subset} of constraints, and each subset of a satisfiable subset is satisfiable as well.
Thus, instead of storing $\setstohit$, we can keep track of a set of \emph{satisfiable subsets} \satsets ; as the intermediate results of calls to \corrsubsets.

When a second call to \comus is performed, we can then initialize $\setstohit$ as the complement of each of these satisfiable subsets with respect to $\formula_2$, i.e.,
 \begin{equation}
 \setstohit\gets \{\formula_2\setminus \m{S}\mid \m{S}\in \satsets\}.
 	\end{equation}
The effect of this is that we \textit{bootstrap} the hitting set solver with an initial set $\setstohit$. 

\subsubsection{Incrementality with MIP}\label{sec:mip-incrementality} 

For hitting set solvers that natively implement incrementality, such as modern Mixed Integer Programming (MIP) solvers, we can generalize this idea further: we know that all calls to $\comus(\formula,f,p)$ will be cast with $\formula \subseteq \m{C}\cup \Iend \cup \overline{\Iend \setminus I_0}$, where $I_0$ is the start interpretation. 
To compute the conditional hitting set for a specific $\formulac\cup I\cup \overline{\Iend\setminus I} \subseteq \m{C}\cup \Iend \cup \overline{\Iend \setminus I_0}$, we need to ensure that the hitting set solver only uses literals in $\formulac\cup I\cup \overline{\Iend\setminus I}$. 
For incremental hitting set solvers, this means updating the constraint $p$ at every explanation step to include (1) only literals from interpretation $I$ at the current explanation step, and (2) the `exactly-one' constraint for explaining one literal at a time.

Since our implementation uses a MIP solver for computing hitting sets (see Section~\ref{sec:background}), and we know the entire formula from which elements must be chosen,
we initialize the MIP solver once with all relevant decision variables of $\m{C}\cup \Iend \cup \overline{\Iend \setminus I_0}$. 

Bear in mind that retracting a constraint $p$ to replace it with an updated one (in the next explanation call) is non-trivial for MIP solvers. 
Therefore, we assign an infinite weight in the cost function to all literals of $\Iend\setminus I$ and update their weights as soon as they have been derived according to the given cost function.
In this way, the MIP solver will automatically maintain and reuse previously found sets-to-hit in each of its computations. 

Next, we investigate how to speed-up the generation of an OCUS using an appropriate \corrsubsets method when domain-specific information is available. Through our running example we will look at the impact of incremetality and a better \corrsubsets procedure.

\subsubsection{Efficiently Generating an Explanation Sequence with Incremental OCUS} \label{sec:ocus-incr}
In the following example, we illustrate the efficiency of \textit{incrementality with MIP} together with the \emph{\sat grow} to speed up generating an OCUS-based explanation sequence.

\begin{example}[\textit{continued}]\label{ex:ocus-incr:expl-seq} 
	Consider the previously introduced clauses of our running example $\formularunning$: 
	 \[ c_1 := \lnot x_1 \vee \lnot x_2 \vee x_3 \qquad  c_2 := \lnot x_1 \vee  x_2 \vee x_3 \qquad  c_3 := x_1 \qquad c_4 := \lnot x_2 \vee \lnot x_3 \]
	To define the input for the MIP-incremental \comus with initial interpretation $\m{I}=\emptyset$, we extend~ $\formularunning$ with the new clauses representing the final interpretation $\Iend =  \{ \{ x_1\},\  \{\lnot x_2\},\ \{ x_3\}\}$ and the complement thereof $\overline{\Iend \setminus \m{I}} =  \{ \{\lnot x_1\},\    \{ x_2\}, \   \{\lnot x_3\}\}$.
	For the MIP-incremental variant of \comus, $p$ remains the same. The cost function $f_I $ is defined as a weighted sum over the following weights:
	\begin{align*}
	\text{\textbf{\emph{Clause weights:}}} & \ w_1 = 60 \  \qquad w_2=60  \qquad  \ \ \ w_3=100 \qquad w_4=100 \\
	\text{\textbf{\emph{$I \wedge \overline{\Iend \setminus I}$ weights:}}}& \  w_{\lnot x_1}=1  \qquad w_{ x_2}=1 \ \ \ \qquad w_{\lnot x_3}=1 \\
	\text{\textbf{\emph{$\Iend \setminus I$ weights:}}}& \  w_{x_1}=\infty  \qquad w_{\lnot x_2}=\infty \qquad w_{x_3}=\infty
	\end{align*}
	Note how the literals that haven't been derived yet ($\Iend \setminus I$) are given an infinite weight according to \cref{sec:mip-incrementality} for incrementality purposes.
	Therefore, $f_I$ will be updated at every explanation step whenever interpretation $I$ changes.
	The incremental \comus-call is now: 	
	\begin{equation}
	\comus( \formulac\land \textbf{\Iend} \land \overline{(\Iend\setminus I)}, f_I, p)		
	\end{equation}
	In this example, the $\grow$ procedure uses the model provided by the \sat solver to grow a given subset $\m{S}$. 
	The \corrsubsets procedure simply returns
	\begin{equation} 
		\m{K} = \corrsubsets(\m{S}, \formula) = \{\formula \setminus \grow(\m{S}, \formula)\}
			\end{equation}

 The following tables (Tables \ref{tab:example-explanation-generation:incr:step1} to \ref{tab:example-explanation-generation:incr:step3}) summarize the intermediate steps to compute an OCUS-based explanation sequence for our running example. 
The literals that cannot be selected by the hitting set solver have been struck out because they have not been derived yet. 

\begin{table}[!ht]
	\centering
	\begin{adjustbox}{max width=\textwidth}
	\begin{tabular}{lcccc} 
		Step & $\m{S}$                             & $\sat(\m{S})$    &           \textit{$\grow(\m{S}, \formula)$}                     & $\setstohit  \gets \setstohit  \cup \m{K}$ \\ 
		\toprule[1pt]
		&    \phantom{$\{ c_1, c_2, x_1, \lnot x_3 \}$}    &    &      & $\emptyset$                                \\

		1. & $\{\lnot x_3\}$                     & $\mathit{true}$  & $\{c_1 , c_2 , c_4, \lnot x_1 , \lnot x_2 ,  \lnot x_3 \}$ & $\{ \{c_3, \cancel{x_1} , x_2 , \cancel{x_3} \}\}$   
		                            \\  		\midrule
		2. & $\{x_2\}$                     & $\mathit{true}$  & $\{c_1, c_2, c_3, x_1, x_2, x_3 \}$ & $\{ ..., \{c_4, \lnot x_1 ,\cancel{\lnot x_2} , \lnot x_3 \}\}$                               \\  
		\midrule 
		3. & $\{ \lnot x_1\}$                    & $\mathit{true}$  & $\{c_1, c_2, c_4, \lnot x_1, \lnot x_2, x_3 \}$            & $\{ ..., \{c_3, \cancel{x_1}, x_2, \lnot x_3 \}\}$                                \\
		\midrule 
		4. & $\{c_4, x_2\}$                      & $\mathit{true}$  & $\{  c_1 , c_2 , c_4, \lnot x_1 , x_2 , \lnot x_3\}$       & $\{ ..., \{ c_3, \cancel{x_1} , \cancel{\lnot x_2} , \cancel{x_3} \}\}$                        \\ 
		\midrule 
		5. & $\{c_3, \lnot x_3\}$                      & $\mathit{true}$  & $\{  c_1, c_3, c_4, x_1 ,\lnot x_2 , \lnot x_3\}$       & $\{ ..., \{c_2, x_2, \cancel{x_3},  \lnot x_1 \}\}$                        \\ 
		\midrule
		6. & $\{c_3, \lnot x_1\}$  & $\mathit{false}$ & - & - \\
	\end{tabular}
\end{adjustbox}
	\caption{\textit{Example (continued)}. Explanation step 1 ($\mathbf{c_3 \implies x_1}$) of the explanation sequence generated with \onestepo with incremental MIP solving (see \cref{sec:mip-incrementality}). }
	\label{tab:example-explanation-generation:incr:step1}
\end{table}
\begin{observation}\textbf{(An effective grow)}
 The most striking aspect of \cref{tab:example-explanation-generation:incr:step1} compared to \cref{tab:example-explanation-generation-not-incremental} is the number of steps required to find an OCUS that is greatly reduced, i.e. from 12 steps to 6, as a result of choosing an effective grow. 
\end{observation}

For the next explanation step, since $x_1$ has been explained, we adapt the weights of the clauses $\{x_1\}$ and $\{\lnot x_1\}$ to $w_{x_1} = 1$ and $w_{\lnot x_1} = \infty$ respectively. 

\begin{observation}\textbf{(Incrementality)}
In the intermediate \comus steps of explanation step 2 (\cref{tab:example-explanation-generation:incr:step2}), we observe the effect of incrementality from the number of steps required to find an OCUS. 
Recall that in the MIP setting,  \formula is constructed overall literals of $\Iend$ and $\overline{\Iend \setminus I}$, and hence stays the same throughout all explanation steps. Therefore, the correction subsets of the previous explanation steps can be reused as is.
Using the previously computed sets-to-hit ensures that the \comus algorithm starts from a good candidate OCUS. If the same \comus-call is performed without bootstrapping the previous sets-to-hit, the number of intermediate steps is higher, i.e. 5 instead of 2.
\end{observation}

\begin{table}[!h]
	\centering
	\begin{adjustbox}{max width=\textwidth}
		\begin{tabular}{lcccc} 
			Step & $\m{S}$                             & $\sat(\m{S})$    &          \grow($\m{S}$, $\formula$)                & $\setstohit  \gets \setstohit  \cup \m{K}      $ \\ \toprule[2pt]
			&    &  &     &  $\{ \{c_3, {x_1}, x_2, \lnot x_3 \}, $                                \\
			&    &  &     &  $\{c_3, {x_1} , x_2 , \cancel{x_3}\}, $                                \\
			&    &  &     &  $\{c_4, \cancel{\lnot x_1} ,\cancel{\lnot x_2} , \lnot x_3\}, $                                \\
			&    &  &     &  $\{c_3, {x_1} , \cancel{\lnot x_2} , \cancel{x_3}\}, $                                \\
			&    &  &     &  $\{c_2, x_2, \cancel{x_3},  \cancel{\lnot x_1}\}\} $                                \\\midrule
			1. &  $\{ c_2 ,x_1, \lnot x_3\}$  & $\mathit{true}$ & $\{ c_2, c_3, c_4,  x_1 , x_2, \lnot x_3 \}$    &  $\{..., \{ c_1, \cancel{x_3},  \cancel{\lnot x_2}\}\}$                                \\
			\midrule 
			2. &  $\{ c_1, c_2, x_1, \lnot x_3 \}$  & $\mathit{false}$ & \phantom{$\{c_1 , c_2 , c_4, \lnot x_1 , \lnot x_2 ,  \lnot x_3 \}$}   &                                 \\
		\end{tabular}
	\end{adjustbox}
	\caption{\textit{Example (continued)}. Explanation step 2 ($\mathbf{c_1 \wedge c_2 \wedge x_1 \implies x_3}$) of the explanation sequence generated with \onestepo \emph{incremental}. }
	\label{tab:example-explanation-generation:incr:step2}
\end{table}

Finally, for the last explanation step, we adapt the weights to reflect the current interpretation and that we only want to explain $\lnot x_2$ ($w_{x_3} = 1$ and $w_{\lnot x_3} = \infty$).

\begin{table}[!h]
	\centering
	\begin{adjustbox}{max width=\textwidth}
	\begin{tabular}{lcccc} 
		Step & $\m{S}$                             & $\sat(\m{S})$    & \grow($\m{S}$, $\formula$)                                 & $\setstohit  \gets \setstohit  \cup \m{K}$ \\ \toprule[2pt]
			&    &  &     &  $\{ \{c_3, {x_1}, x_2,\cancel{\lnot x_3} \}, $                                \\
			&    &  &     &  $\{c_3, {x_1} , x_2 , {x_3}\}, $                                \\
			&    &  &     &  $\{c_4, \cancel{\lnot x_1} ,\cancel{\lnot x_2} , \cancel{\lnot x_3}\}, $                                \\
			&    &  &     &  $\{c_3, {x_1} , \cancel{\lnot x_2} , {x_3}\}, $                                \\
			&    &  &     &  $\{c_2, x_2, {x_3},  \cancel{\lnot x_1}\}, $                                \\
			&    &  &     &  $\{ c_1, {x_3},  \cancel{\lnot x_2}\}\}$                                \\\midrule
		1. &  $\{ c_4, x_2, x_3  \}$  & $\mathit{false}$ & \phantom{$\{ c_1,  c_3, c_4, x_1, \lnot x_2,\lnot x_3 \}$}    & \phantom{$\{..., \{ c_2, \lnot x_2, \lnot x_3\}\}$ }\\
	\end{tabular}
	\end{adjustbox}
	\caption{\textit{Example (continued)}. Explanation step 3 ($\mathbf{c_4 \wedge x_3 \implies \lnot x_2}$) of the explanation sequence generated with \onestepo \emph{incremental}. }
	\label{tab:example-explanation-generation:incr:step3}
\end{table}

\begin{observation}\textbf{(Disjoint Correction Subsets)} Observe set-to-hit $\{ c_4, \cancel{\lnot x_1}, \cancel{\lnot x_2}, \cancel{\lnot x_3}\}$ in the collection of previously compute sets-to-hit of \cref{tab:example-explanation-generation:incr:step3}. Given the current interpretation and the literal $\lnot x_2$ to explain, only $c_4$ can \underline{and} has to be taken. 
The phenomenon of a set being disjoint from another with respect to $p$ is what we call in \cref{prop:p-disjoint} $p$-disjointness.
In this case, subset $\{ c_4, \cancel{\lnot x_1}, \cancel{\lnot x_2}, \cancel{\lnot x_3}\}$ is $p$-disjoint from the other sets-to-hit for the hitting set solver. Therefore, it poses a stronger restriction on the sets-to-hit. 
\end{observation}
\end{example}

\begin{definition}\label{prop:p-disjoint}
	Two sets $S_1$ and $S_2$ are $p$-disjoint if every set that hits both $S_1$ and $S_2$ and satisfies $p$ contains $s_1\in S_1$ and $s_2\in S_2$ with $s_1 \neq s_2$.
\end{definition}

\cref{ex:ocus-incr:expl-seq} shows that \emph{incrementality with MIP} and \emph{growing the satisfiable subset} $\m{S}$ are effective at reducing the size and the number of sets-to-hit when computing a sequence of explanations using \onestepo. 
Next, in section \ref{sec:correction_subset_enum}, we take advantage of \cref{prop:p-disjoint} to enumerate multiple correction subsets that are $p$-disjoint of each other during the \corrsubsets procedure.

\subsubsection{Correction Subsets Enumeration}\label{sec:correction_subset_enum}

Our \comus algorithm repeatedly alternates between computing hitting sets and correction subsets. 
The increasingly large collection of sets-to-hit makes finding optimal hitting sets much more expensive compared to the \corrsubsets procedure, which solely relies on finding correction subsets from a given satisfiable subset.
Additionally, in the last explanation step of \cref{ex:ocus-incr:expl-seq}, we saw that one of the sets-to-hit ($\{c_4, \cancel{\lnot x_1} ,\cancel{\lnot x_2} , \cancel{\lnot x_3}\}$) was $p$-disjoint to the others, imposing a strong restriction on the hitting set solver. 
The question is: 
\begin{quote}
Can we cheaply find \textit{multiple}, ideally \emph{$p$-disjoint}, correction subsets and thereby add multiple sets-to-hit in one go?
\end{quote}

Inspired by \citet{marques2013computing}, we depict a \corrsubsets procedure
in \cref{alg:corr-subsets-grow-with-sat} that computes a collection of correction subsets $\m{K}$ starting from a given subset of constraints $\m{S}$ (either the empty set or a computed hitting set). 
The \corrsubsets procedure will repeatedly compute a satisfiable subset $\m{S}' \supseteq \m{S}$, and add its complement 
to the collection of disjoint MCSes $\m{K}$ and to $\m{S}$, ensuring that the constraints in the correction subset cannot be present in the next correction subset (disjoint), until no more satisfiable subsets can be found. 

\begin{algorithm}[!h]
	\DontPrintSemicolon
	$\m{K} \gets \emptyset$\;
	$\subsetT' \gets  \subsetT$ \label{alg:disjmcs:init} \;
	
	\While{$\sat(\m{S}')$}{
		C $\gets  \formula \setminus \grow(\m{S}', \formula)$ \label{alg:corr-subsets-grow-with-sat:corr-subset}\;
		$ \m{S}' \gets \m{S}' \ \cup $ C \label{alg:corr-subsets-grow-with-sat:mapping}\;
		$ \m{K} \gets \m{K} \ \cup \{$C$\}$\;
	}
	\Return $\m{K}$
	\caption{$\corrsubsets(\subsetT, \formula)$ }
	\label{alg:corr-subsets-grow-with-sat}
\end{algorithm}

For simple constraints $p$, \cref{alg:corr-subsets-grow-with-sat} is directly applicable and will be able to compute \emph{disjoint} correction subsets.
However, for more complicated $p$ constraints, this easily degrades into computing only a single correction subset.

\subsubsection{Correction Subsets Enumeration with Incremental MIP} If \cref{alg:corr-subsets-grow-with-sat} is directly applied to the MIP incremental \comus variant that considers the whole formula $\formulac\land \textbf{\Iend} \land \overline{(\Iend\setminus I)}$, both the literal-to-explain and its negation will be present at line~\ref{alg:corr-subsets-grow-with-sat:mapping} leading to UNSAT. In a given explanation step, only the base constraints $\formulac$, the current interpretation $I$, and the literals of $\overline{(\Iend\setminus I)}$ can be hit by the hitting set solver.  
Therefore, we project subset $\m{S'}$ onto the base constraints, the current interpretation, and the negated literals to explain. This extra step is executed right after line~\ref{alg:corr-subsets-grow-with-sat:mapping} of \cref{alg:corr-subsets-grow-with-sat}:
\[ \m{S}' \gets \m{S}'  \cap (\formulac \cup I \cup \overline{\Iend \setminus I}) \]
By incorporating \emph{explanation-specific} information, we are able to enumerate extra correction subsets that are $p$-disjoint.

\begin{example}[\cref{ex:ocus-incr:expl-seq} continued] \label{ex:correction-subset-enum}
	Consider the same setting as in \cref{ex:ocus-incr:expl-seq} with given initial interpretation $\m{I} =\emptyset$ and $\Iend = \{x_1, \lnot x_2, x_3\}$.
	\cref{tab:example-corr-subsets:ocus-incr} illustrates the efficiency of the incremental variant of the \corrsubsets algorithm starting from hitting set $\m{S} := \{\lnot x_3\}$ as in the first step of our running example (\cref{tab:example-explanation-generation:incr:step1} of \cref{ex:ocus-incr:expl-seq}). This example uses the \sat-based \grow of \cref{para:domainspecificgrow}. 
	
	\begin{table}[!h]
		\centering
		\begin{tabular}{lcccc} 
			Step & $\m{S}'$                             & $\sat(\m{S}')$    & $\m{S}' \gets \grow(\m{S}, \formula)$      & $\m{K}  \gets \m{K}  \cup \{  \formula \setminus \m{S}'\}$ \\ 
			\toprule[2pt]
			&    \phantom{$\{ c_1, c_2, x_1, \lnot x_3 \}$}    &    &      & $\emptyset$                                \\\midrule
			1. & $\{ \lnot x_3\}$                    & $\mathit{true}$  & $\{c_1, c_2, c_4, \lnot x_1, \lnot x_2, \lnot x_3 \}$            & $\{ \{c_3, \cancel{x_1}, x_2, \cancel{x_3}\}\}$                                \\\midrule
			2. & $\{c_3, x_2, \lnot x_3\}$ & $\mathit{true}$  & $\{c_2 , c_3, c_4, x_1 , x_2, \lnot x_3 \}$ & $\{ ..., \{c_1, \lnot x_1, \cancel{\lnot x_2}, \cancel{x_3}\}\}$                               \\\midrule
			3 & $\{c_1, c_3, \lnot x_1, x_2, \lnot x_3\}$& $\mathit{false}$  &  &                       \\  
		\end{tabular}
		\caption{\textit{\cref{ex:ocus-incr:expl-seq} (continued)}. Correction Subset enumeration starting from hitting set $\{ \lnot x_3\}$ with \emph{MIP-incremental} \comus.}
		\label{tab:example-corr-subsets:ocus-incr}
	\end{table}

	The MIP-based incremental variant of \comus uses unsatisfiable formula \formula defined as $\formulac\land \textbf{\Iend} \land \overline{(\Iend\setminus I)}$. For given interpretation $I = \emptyset$, the $\formula$ corresponds $\formula := \formulac\land \{ x_1\}\land  \{\lnot x_2\}\land \{ x_3\}\land \{ \lnot x_1\}\land  \{ x_2\}\land \{ \lnot x_3\}$.

	Taking a closer look at step 1 of \cref{tab:example-corr-subsets:ocus-incr}, $\lnot x_3$ is present in $\m{S}$ and $x_3$ is in the corresponding correction subset. If we were to add it directly to $\m{S}$, the algorithm would stop at the next iteration since $\m{S}$ would be unsatisfiable. However, $x_3$ (and $x_1$) cannot be selected by the hitting set solver, and therefore should not be added. 
	By adding only the clauses from $\{\formula \setminus \m{S}\}$ projected onto $ (\formulac\cup I \cup \overline{\Iend \setminus I})$, i.e.~$c_3$ and $x_2$, we are able to enumerate \textbf{multiple} correction subsets that are $p$-disjoint for the hitting set solver. 
\end{example}

\cref{ex:correction-subset-enum} depicts the effectiveness of enumerating multiple correction subsets that are $p$-disjoint by projecting the correction subset onto the unsatisfiable formula for the current interpretation $(\formulac\cup I \cup \overline{\Iend \setminus I})$. 
In the rest of the paper, we consider correction subset enumeration with a \sat-based grow as the baseline approach. We refer to it as \emph{Multi-\sat}.


\section{Experiments}\label{sec:experiments}

In this section, we validate the \emph{qualitative improvement} of computing explanations that are optimal with respect to a cost function as well as the \emph{performance improvement} of the different versions of our algorithms.

\subsection*{Experimental Setup}

Our experiments\footnote{The code for all experiments is made available at \texttt{https://github.com/ML-KULeuven/ocus-explain}.} were run on a compute cluster where each explanation sequence was assigned a single core on a 10-core INTEL Xeon Gold 61482 (Skylake) processor with a time limit of 60 minutes and a memory-limit of 8GB. The code is written on top of PySAT 0.1.7.dev1 \cite{pysat}.
For the MIP calls, we used Gurobi 9.1.2, for SAT calls MiniSat 2.2 and for \maxsat~calls RC2 as bundled with PySAT.
In the MUS-based approach, we used PySAT’s deletion-based MUS extractor MUSX~\cite{marques2010minimal}. 

Regarding the benchmark dataset, we rely on a set of generated Sudoku puzzles of increasing difficulty (different amount of given numbers), the Logic Grid puzzles of \citet{bogaerts2021framework}, as well as the instances from \citet{schotten}.\footnote{We express our gratitude towards \emph{Matthew.~J.~McIlree} and \emph{Christopher Jefferson} of St.~Andrews University for their help with the extraction of CNF instances from Essence problem specifications \cite{frisch2008essence} using Savile Row \cite{nightingale2017automatically}, as well as for supplying the problem instances.}

When generating an explanation sequence for these puzzles, the unsatisfiable subsets identify which constraints and which previously derived facts should be combined to derive new information. 
Similar to \citet{ijcai2021}, for Logic Grid puzzles, we assign a cost of 60 for puzzle-agnostic constraints; 100 for puzzle-specific constraints; and a cost of 1 for facts. 
For all other puzzles, we assign a cost of 60 when using a constraint and a cost of 1 for facts.
In  \cref{tab:puzzle-composition-family} of \cref{apx:puzzledata}, we summarize the average number of clauses (avg. $\#$ {clauses}), the average number of literals to explain (avg. $\#$ lits-to-explain) as well as the number puzzles (n) in the benchmark data set for each puzzle family. 
\paragraph{Research Questions} Our experiments are designed to answer the following research questions:  

\begin{enumerate}[label={Q\arabic*}, ref={Q\arabic*}]
\item  What is the effect of using an \textit{optimal} unsatisfiable subset, on the quality of the generated step-wise explanations? \label{rq:q1}
\item What is the impact of information \textbf{re-use} on the efficiency of \comus? \label{rq:information-reuse}
\item What are the time-critical components of \comus?\label{rq:time-critical}
\item  How does more advanced extraction of correction subsets and extraction of multiple correction subsets affect performance?\label{rq:q3}
\item   What is the efficiency of a single step O(C)US

\begin{enumerate}
	\item from an instantaneous (time-to-first) explanation  point of view?
    \item from a step-wise (single next explanation) solving point of view?
\end{enumerate} \label{rq:q4}
\end{enumerate}

\subsection{Explanation Quality}
To evaluate the effect of optimality on the quality of the generated explanations, we reimplemented a MUS-based explanation generator based on \cref{alg:oneStep}. 
Before presenting the results, we want to stress that this is \emph{not} a fair comparison with the implementations of \citet{ecai/BogaertsGCG20} and \citet{schotten}, where a heuristic is used that relies on \emph{even more} calls to \call{MUS} in order to avoid the quality problems we will illustrate below.
While in both cases this would yield better explanations, it comes at the expense of computation time, thereby leading to several hours to generate the explanation of a single puzzle. 

To answer \textbf{Q1}, we ran the OCUS-based algorithm as described in \cref{alg:oneStepOCUS} and compared 
at every step the cost of the produced explanation with the cost of MUS-based explanation of \cref{alg:oneStep}.
These costs are plotted on a heatmap in Figure \ref{fig:rq1_heatmap}, where the darkness represents the number of occurrences of the combination at hand. 

We see that the difference in quality is striking in many cases, with the MUS-based solution often missing very cheap explanations (as seen by the darkest squares in the column above cost 60), thereby confirming the need for a cost-based OUS/OCUS approach.

\begin{figure}[!h]
		\centering
		\includegraphics[width=0.54\textwidth]{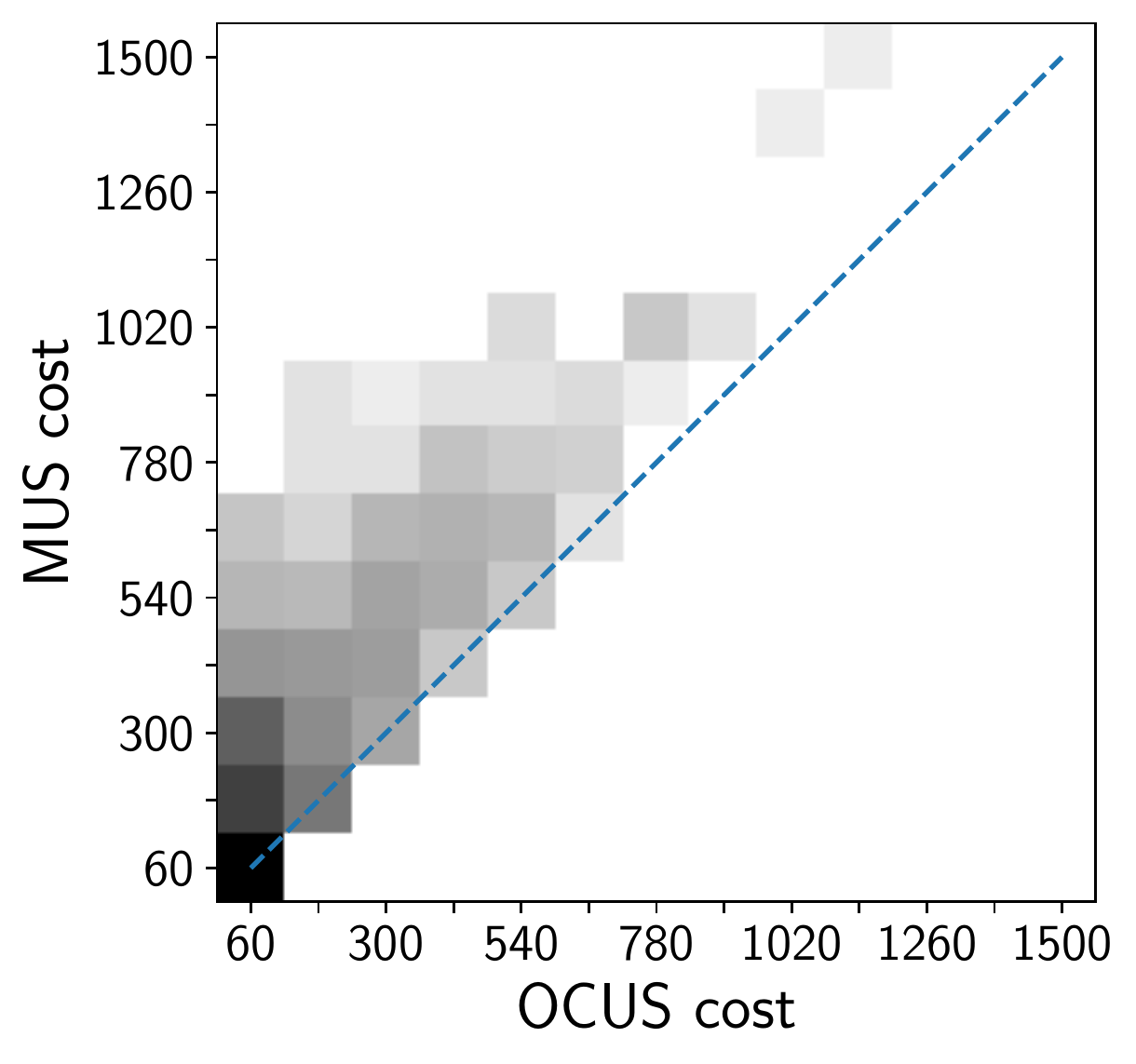}
		\caption{Q1 - Explanation quality comparison of optimal versus subset-minimal explanations in the generated puzzle explanation sequences.}
		\label{fig:rq1_heatmap}
	\end{figure}

\subsection{Information Re-use}\label{sec:experiments:incrementality}

To answer \textbf{\ref{rq:information-reuse}}, we compare the effect of incrementality when generating a sequence of explanations. 
Next to \begin{inparaenum}[(1)]
	\item \comus, we also include
	\item the \emph{bounded} \comus (\ousb) algorithm, where we call the \comus algorithm for every literal in every step, but we reuse information by giving it the current best bound $f(S_{best})$ and iterating over the literals that performed best in the previous call first; and
	\item the \emph{split} \comus (\ousbiter) approach, where we split up the computations by iteratively selecting the most promising literal and expanding only one hitting set for it, then re-evaluating the most promising literal and so forth.
\end{inparaenum}

\subsubsection{Incremental versus Non-incremental Explanation Generation}

Preliminary experiments have shown that incrementality by keeping track of the hitting sets using the MIP solver significantly outperforms bootstrapping satisfiable subsets. 
Therefore, in the experiments, we only show results on incrementality using MIP solvers.

For \comus, incrementality (\emph{+Incr}) is achieved by reusing the same MIP hitting set solver throughout the explanation calls, as explained in \cref{sec:ocusEx}.
Methods \ousb and \ousbiter use a separate hitting set solver for every literal to explain. Each hitting set solver can similarly be made incremental (with respect to its own literal) across the explanation calls (\emph{+Incr.}), or not.

Figure \ref{fig:rq2_cactus_incrementality_log} depicts the number of instances fully explained through time before the given 1-hour timeout. In each configuration depicted, we use a single \sat-based \grow in the \corrsubsets procedure. The same color is used for the same \comus algorithms, with a full line for the incremental version and a dashed line for the non-incremental one.

\begin{figure}[!h]
	\centering
	\includegraphics[width=0.9\textwidth]{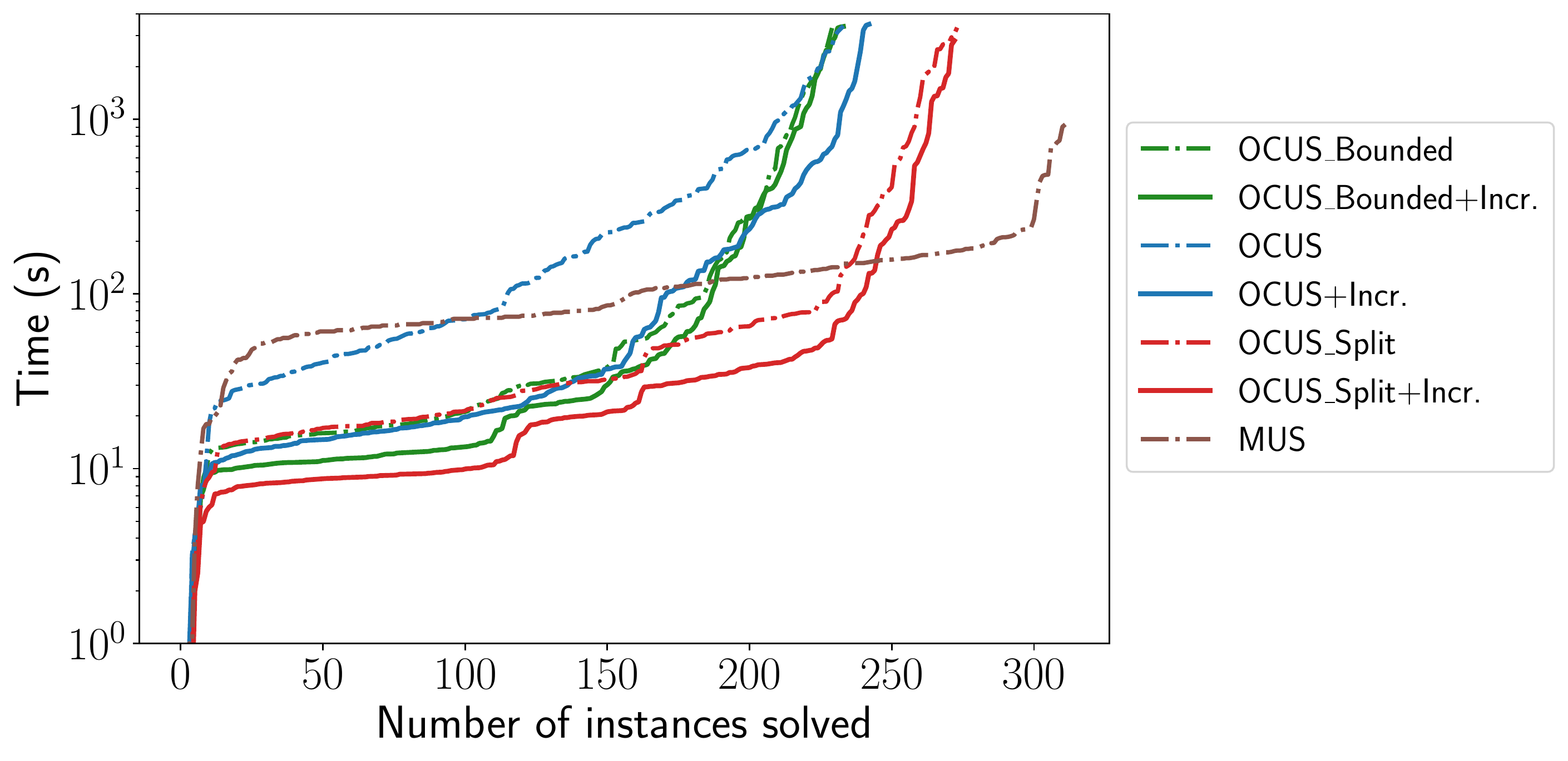}
	\caption{Time to generate a full explanation sequence with a single \sat grow call in the \corrsubsets procedure for incremental and non-incremental \comus algorithms. Incrementality improves the number of instances solved and the time required to solve them for all O(C)US algorithms.}
	\label{fig:rq2_cactus_incrementality_log}
\end{figure}

If we compare all configurations, we see that \comus is faster than the plain \mus-based implementation for simpler instances. \mus is able to explain more instances than all \comus versions, as it solves a simpler problem (with worse quality results as shown in \textbf{Q1}).
\comusbound is faster than \comus for easier instances but explains about the same number of instances, and \comussplit is considerably faster and solves the most instances out of all \comus algorithms.

The effect of introducing incrementality produces a speed-up in \comus. For \comusbound it is negligible and for \comussplit there is a speed-up in all but its most difficult instances.
In general, we see that the curves of the incremental variants are located somewhat lower.
The best runtimes are obtained with \ousbiterincr, that is, using an incremental hitting set solver \emph{for every individual literal-to-explain separately} and only expanding the literal that is most likely to provide a cost-minimal OCUS.

\subsubsection{Instance-level Speed-up with Incrementality}\label{sec:Instance-level speed-up with incrementality} Next, we analyze the speed-up of introducing incrementality per instance. \cref{fig:rq2_runtime_incr_non_incr} compares the time to generate a full sequence for the incremental variant with the non-incremental version for each puzzle of every \comus configuration. Results in the upper triangle signify an improvement using incrementality, and the lower triangle indicates worse performance with incrementality.

	\begin{figure}
	\centering
	\includegraphics[width=0.54\textwidth]{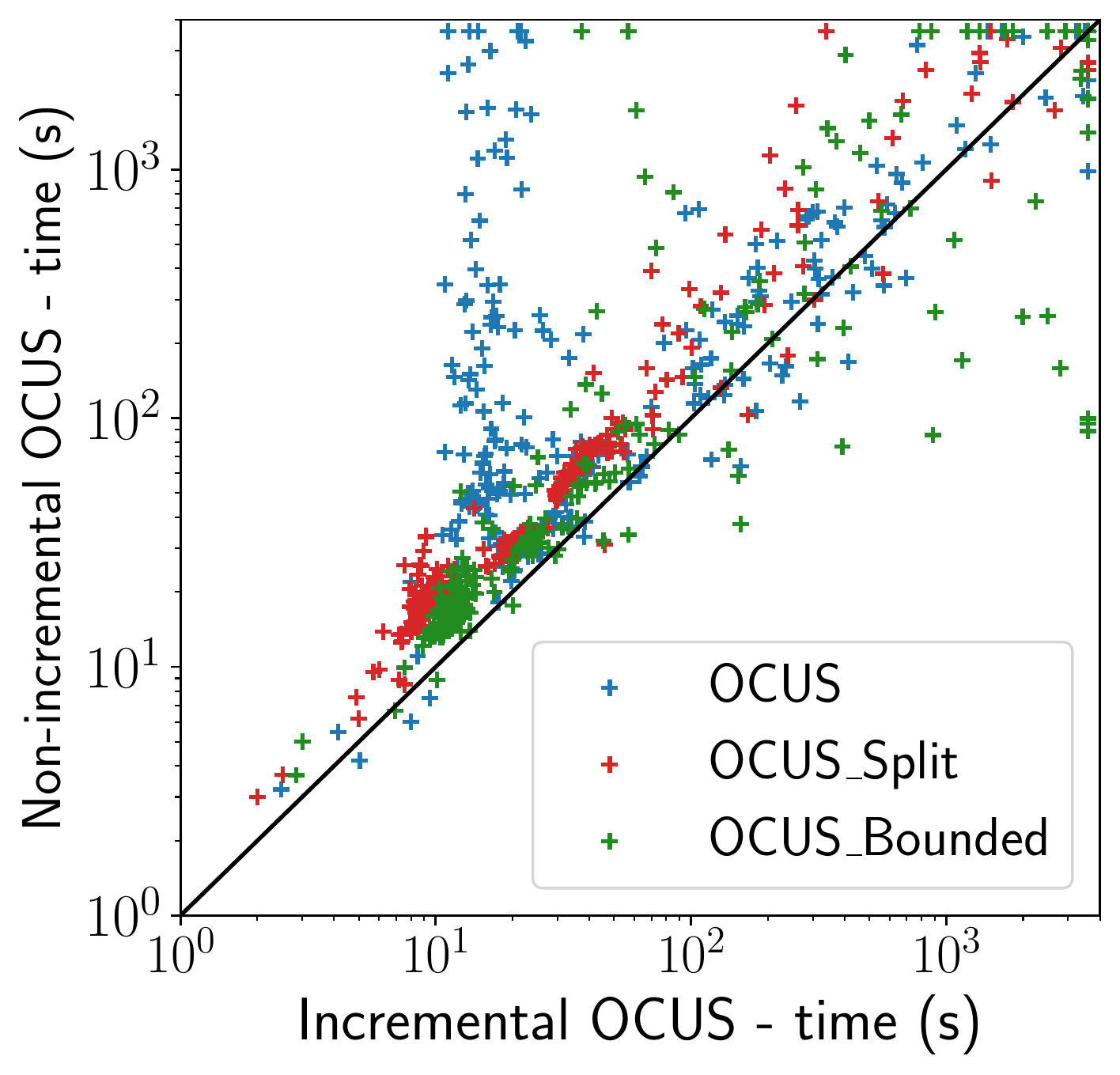}
	\caption{Q2 - Runtime comparison of introducing incrementality on the time to generate a whole explanation sequence with \sat-based \grow.}
	\label{fig:rq2_runtime_incr_non_incr}
\end{figure}

Taking a closer look at the results for \comus, the story is similar to the results of \cref{fig:rq2_cactus_incrementality_log}. Many of the instances are explained quickly (around 10 to 20 seconds) by \comusincr, whereas non-incremental \comus takes much longer to explain some instances, even leading to a timeout.

Incrementality with \comussplit and \comusbound produce marginal improvements on the runtime. Most of the instances lie close to the black line.
Only a few instances are present in the lower triangle, for the incremental variant of \comusbound. For these instances, \ousbincr takes on average double the time to generate the first explanation step due to the overhead associated with introducing incrementality. The story is similar for the few instances where \ousbiterincr is slower than \comussplit.

\subsection{Runtime Decomposition} To evaluate the time-critical components of \comus (Q3), we decompose the time spent in each part of the \call{OCUS} algorithms in \cref{tab:runtime-decomposition}. We only included the runtimes of instances that do not time out for all methods to allow a fair comparison. 

	The table depicts for each configuration 
\begin{inparaenum}[(1)]
	\item the number of instances \emph{explained};
	\item \%\texttt{OPT} the percentage of time spent in the hitting set solver;
	\item \%\sat the percentage of time spent in the sat solver:
	\item \%\texttt{CorrSS} the percentage of time spent in the \corrsubsets procedure; and
	\item $N_{sth}$ the average total number of sets-to-hit computed.
\end{inparaenum}
Each version of the \comus algorithm uses the \sat-based grow in the \corrsubsets procedure (see \cref{tab:correction-subset-proc-description}).

Looking closer at the runtime decomposition, we observe that most of the time in all the \comus algorithms is spent computing the optimal hitting sets.
The results in the $N_{sth}$ column (top 3 rows versus incremental bottom 3 rows), highlight the decrease in the amount of sets-to-hit that need to be computed when adding incrementality, in our case, re-using the MIP solver between explanation steps. 
Indeed, similar to the results of \cref{ex:ocus-incr:expl-seq}, the \comus algorithms do not need to recompute the previously derived sets-to-hit, and will therefore start with a good candidate hitting set at the beginning of an explanation step.

\begin{table}[!h]
	\centering
	\begin{adjustbox}{max width=\textwidth}
		\begin{tabular}{r|c|c|c|c|c}
			\textbf{config}             &  explained  &  \%\texttt{OPT} &   \%\sat & \%\corrss &  $N_{sth}$ \\ \midrule
			\comus                      & [233 / 403] &          96.65\% &  3.01\% & 0.35\% &      6245     \\
			\comusbound                 & [229 / 403] &          86.86\% & 12.31\% & 0.83\% &      9310     \\
			\ousbiter                       & [273 / 403] &          88.58\% & 10.41\% & 1.01\% &      7480     \\\midrule
			\comusincr                  & [242 / 403] &          99.16\% &  0.78\% & 0.06\% &       405     \\
			\ousbincr  			        & [233 / 403] &          95.95\% &  3.94\% &  0.1\% &      1538     \\
			\ousbiterincr               & [272 / 403] &          96.28\% &  3.42\% &  0.3\% &      1713     \\\midrule
			MUS                          & [311 / 403] &              --- &     --- &    --- &       ---    \\
		\end{tabular}
	\end{adjustbox}
	\caption{Runtime decomposition of core parts of \call{O(C)US} with the \sat-based grow. On the left, most of the computational time is spent in computing more optimal (constrained) hitting sets, while on the right, the runtime is shifted towards higher quality  correction subsets. Incrementality, by re-using the same MIP solver throughout the explanation steps, helps to drastically reduce the average number of sets-to-hit $N_{sth}$.}
	\label{tab:runtime-decomposition}
\end{table}

\subsection{Correction Subset Extraction}

	As \cref{tab:runtime-decomposition} table shows, most time is spent computing the optimal hitting sets. Hence, there is potential in reducing its effort by computing better or more correction subsets, thereby having better collections of sets-to-hit.
This induces a trade-off between the \emph{efficiency} of the \grow strategy, the \emph{quality} of the produced \emph{satisfiable subset} as well as the corresponding \emph{sets-to-hit} generated by the \corrsubsets procedure. 
In this section, we evaluate whether an efficient \grow-call is able to balance the efficiency and the quality of the produced satisfiable subset. Second, we analyze whether the enumeration of multiple, ideally $p$-disjoint, correction subsets 
reduces the time spent in the hitting set solver.

Thus, to answer \textbf{\ref{rq:q3}}, we compare the incremental variants of \comus, \emph{bounded} \comus, \emph{split} \comus, and only change the \corrsubsets strategy they use.

\paragraph{Efficient Correction Subset Enumeration} 
Extending the observation of \cref{para:domainspecificgrow} that a call to an efficient \maxsat~solver helps in finding maximally satisfiable subsets, we propose three additional variants of \cref{alg:corr-subsets-grow-with-sat}:
\begin{enumerate}
	\item The first variant, \emph{Multi-\sat}, repeatedly grows with \sat and blocks the corresponding correction subset until no more correction subsets can be extracted.
	\item The second variant is \emph{Multi-\maxsat}. This correction subset enumerator repeatedly grows using the domain-specific \maxsat~introduced in \cref{para:domainspecificgrow} until the updated subset $\m{S}'$ is no longer satisfiable. 
	\item The last variant \emph{Multi-\subsetmaxsat} repeatedly grows using the \emph{\subsetmaxsat} to balance the trade-off between efficiency and quality.
\end{enumerate}

\begin{table}[!h]
	\centering
	\begin{adjustbox}{max width=\textwidth}
		
		\begin{tabular}{|l|l|}
			\hline
			\textbf{\corrsubsets}  &  \textbf{Description}\\\hline
			
			\sat & \textit{Extract a satisfiable model computed by the \sat solver.} \\\hline
			\multirow{2}{*}{\subsetmaxsat} & \textit{Extends the satisfiable subset computed by \sat by adding every} \\
			&			\textit{remaining clause $c \in \formula \setminus \m{S}$ if $\m{S} \cup \{c\}$ is satisfiable.} \\\hline
			\multirow{3}{*}{\domspecmaxsat} & \textit{Grow using an unweighted \maxsat~solver with domain-specific }\\
			&  \textit{information, i.e. only previously derived facts and the original}\\
			&  \textit{constraints (see Section~\ref{sec:ocusEx}).}\\\hline
			\multirow{2}{*}{\maxsatfull} & \textit{Grow with an unweighted \maxsat~solver using the full } \\
			& \textit{unsatisfiable formula $\mathcal{F}$.} \\\hline
			\multirow{2}{*}{\emph{Multi}~\sat}& \textit{Repeatedly grow with \textbf{\sat}  and block the corresponding correction}\\ 
			& \textit{subset until no more correction subsets can be extracted.}\\\hline
			\multirow{2}{*}{\emph{Multi}~\domspecmaxsat}& \textit{Repeatedly with \textbf{\emph{\domspecmaxsat}} and block the corresponding}\\ 
			& \textit{correction subset until no more correction subsets can be extracted.}\\\hline
			\multirow{2}{*}{\emph{Multi}~\subsetmaxsat} &	\textit{Repeatedly grow with \textbf{\emph{\subsetmaxsat}}  and block the corresponding}\\ 
			& \textit{correction subset until no more correction subsets can be extracted.}\\\hline
		\end{tabular}
	\end{adjustbox}
	\caption{Description of the correction subset procedures.}
	\label{tab:correction-subset-proc-description}
\end{table}

\begin{figure}[!ht]
	\begin{subfigure}{0.5\textwidth}
		\centering
		\includegraphics[width=\textwidth]{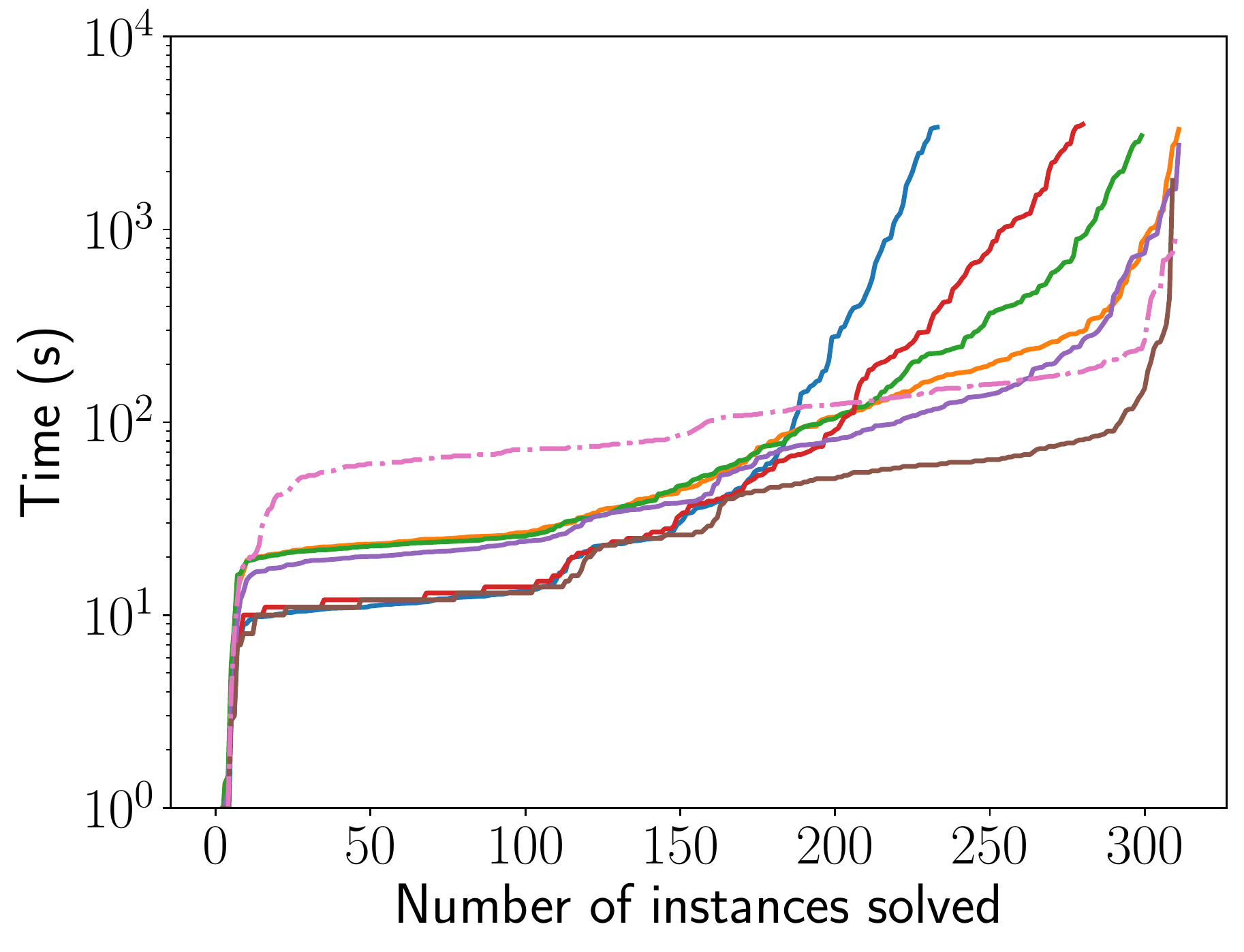}
		\caption{\ousbincr}
		\label{fig:rq3_grow:ousb_iter}
	\end{subfigure}
	\begin{subfigure}{0.5\textwidth}
		\centering
		\includegraphics[width=\textwidth]{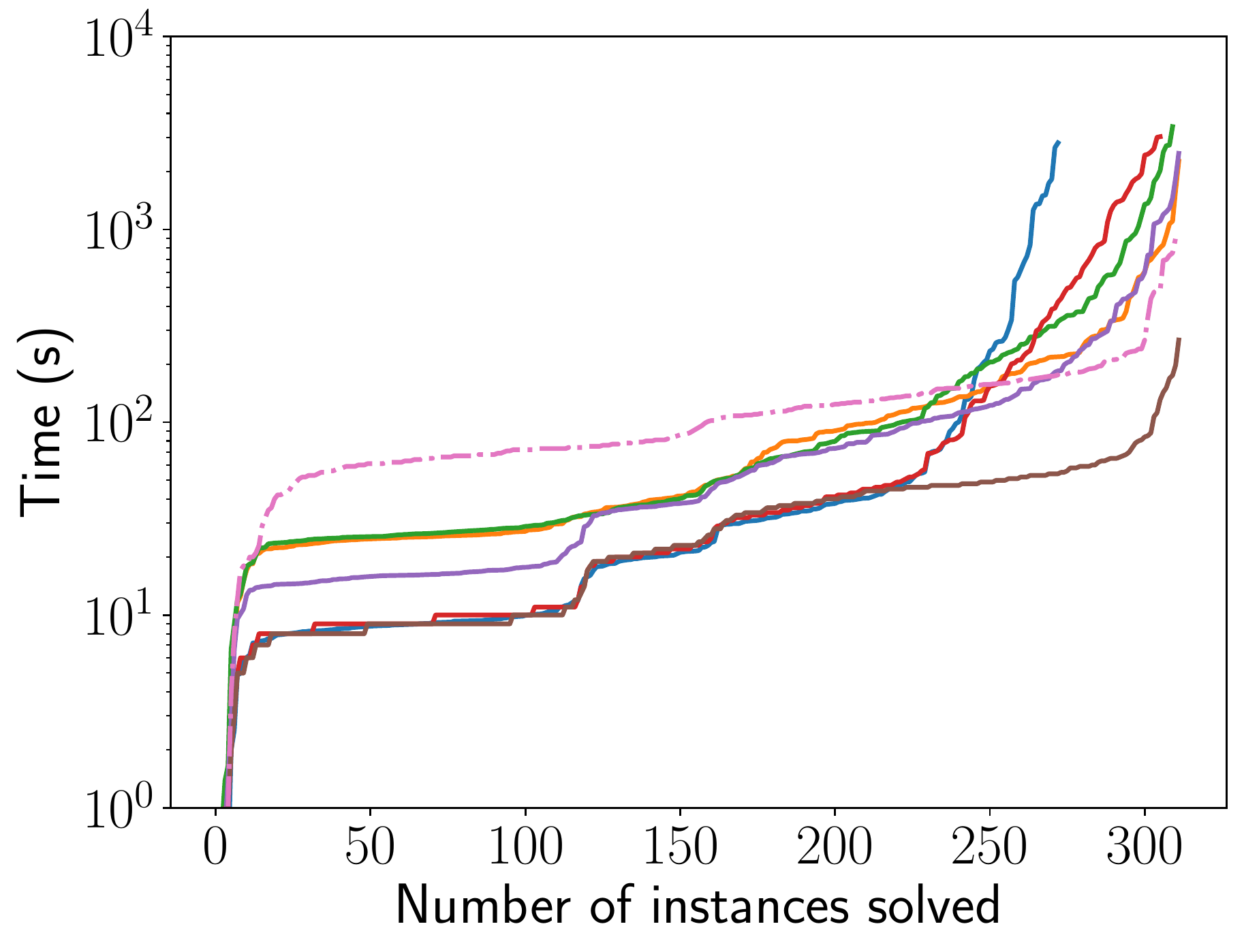}
		\caption{\ousbiterincr}
		\label{fig:rq3_grow:ous_iter}
	\end{subfigure}
	\\
	\begin{subfigure}{\textwidth}
		
		\includegraphics[width=0.93\textwidth]{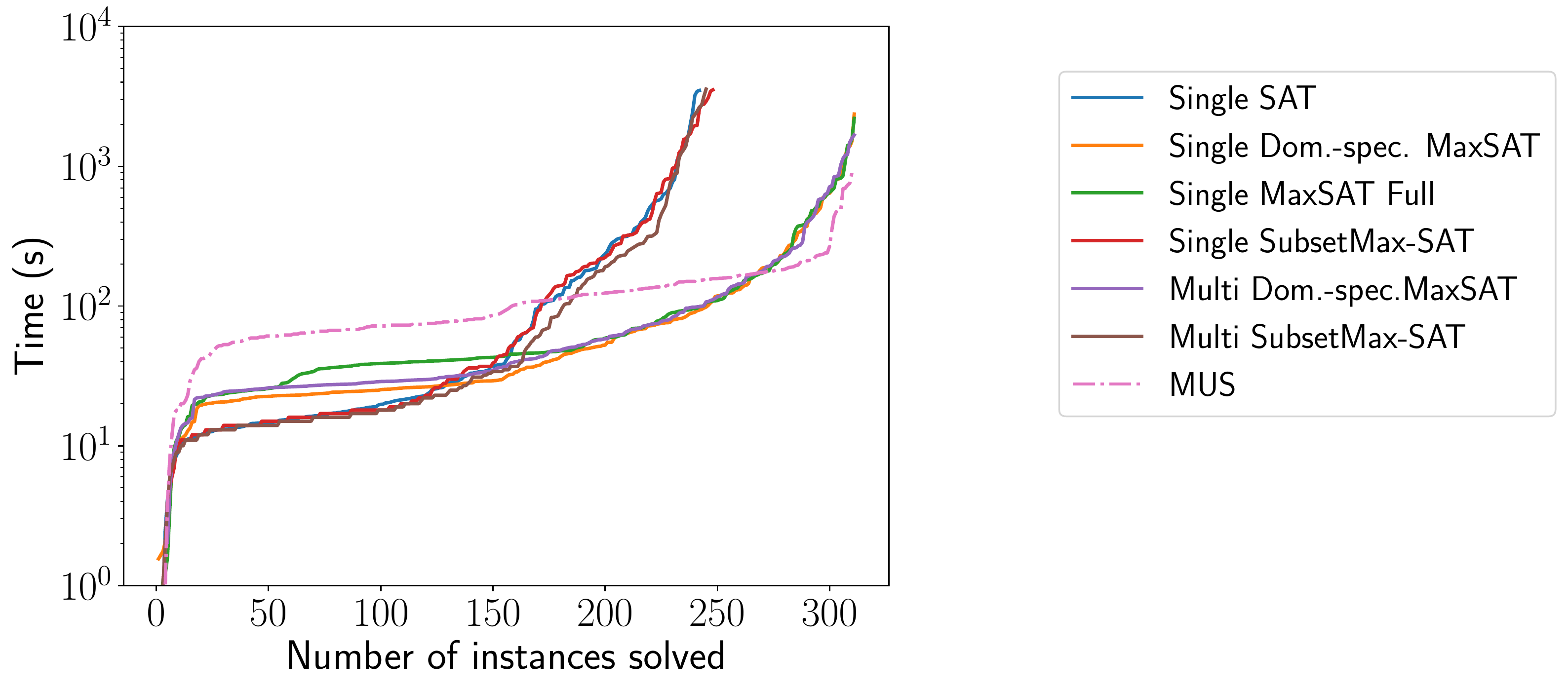}
		\caption{\comusincr\hphantom{-----------------------------------------------------------}}
		\label{fig:rq3_grow:ocus_incr}
	\end{subfigure}
	\caption{Average cumulative explanation time of correction subset configurations for the incremental O(C)US variants. The legends for each graph are ordered from the least instances explained (top) to the most instances explained (bottom).}
	\label{fig:q3:grow_configurations}
\end{figure}

Figure \ref{fig:q3:grow_configurations} depicts for each incremental \comus algorithm, the number of instances solved across time for varying correction subset approaches. 
For \comusincr, all \maxsat-based methods for enumerating correction subsets provide similar results and are able to explain more instances than the rest. 
The story for \ousbiterincr and \ousbincr is different. The best method for enumerating correction subsets with \ousbiterincr and \ousbincr is \multisubsetmaxsat. \multisubsetmaxsat is able to consistently beat all other correction subset enumeration methods and is substantially faster than the naive MUS approach.
To summarize, the best explanation sequence configuration corresponds to the iterated MIP-incremental approach of \ousbiterincr combined with \emph{Multi}-\subsetmaxsat.

\paragraph{Runtime Decomposition} We now look at the runtime decomposition in \cref{tab:runtime-decomposition-corr-susbet-enum} for the best performing \corrsubsets configuration, i.e.~\multisubsetmaxsat, also for the non-incremental OCUS variants. 
For the non-incremental variants, we observe that the shift in computation from the hitting set solver to correction subset enumeration clearly reaps its benefits.

\begin{table}[!h]
	\centering
	\begin{tabular}{r|c|c|c|c|c}
		&  \multicolumn{5}{c}{\textbf{\emph{Multi}-\subsetmaxsat}}\\[1ex]
		\textbf{config}          & \textbf{explained}  &  \textbf{\%\texttt{OPT}} &   \textbf{\%\sat} & \textbf{\%\corrss} &  $\mathbf{N_{sth}}$   \\ \midrule 
		\comus                  & [291 / 403] & 68.15\% & 2.78\%  &    29.07\% &      3649 \\
		\comusbound             & [311 / 403] & 39.87\% & 6.53\%  &     53.6\% &     29747 \\
		\ousbiter               & [311 / 403] & 45.24\% & 3.62\%  &    51.14\% &     28320 \\ \hline
		\comusincr              & [245 / 403] & 87.38\% & 1.03\%  &    11.59\% &       402 \\
		\ousbincr  		        & [309 / 403] & 81.57\% & 5.68\%  &    12.76\% &      1251 \\
		\ousbiterincr           & [311 / 403] & 79.64\% & 1.77\%  &    18.58\% &      1274 \\ \hline
		MUS                     & [311 / 403] &    --- &       ---&          ---&      --- \\
	\end{tabular}
	\caption{Runtime decomposition of core parts of \call{OCUS}. On the left, most of the computation time is spent in computing more optimal (constrained) hitting sets, while on the right, the runtime is shifted toward higher-quality correction subsets.}
	\label{tab:runtime-decomposition-corr-susbet-enum}
\end{table}

Compared to \cref{tab:runtime-decomposition}, more instances are explained, and the computation is more balanced both for incremental and non-incremental \comus algorithms.
The non-incremental \comus variants require considerably more correction subsets to be extracted and somewhat fewer in the incremental case. 
However, the number of instances explained is substantially higher for all \comus configurations.

\subsection{Efficiency of Single Step O(C)US, Instantaneous and Step-wise}

In this section, we analyze whether the algorithms we developed could be fit for an interactive context. By an interactive context, we consider two types of settings. First, a user may want an \emph{immediate explanation step} for the given CSP, and second, a user is asking for the next step, and the next, and so forth. 

\cref{tab:indiv-explanations-all} depicts \emph{Multi}-\subsetmaxsat~the best performing \corrsubsets procedure of \cref{fig:q3:grow_configurations}. For each \comus algorithm, we compute the average time to produce the first explanation step $\overline{t_{1}}$ (instantaneous), the average time to generate one step-wise explanation $\overline{t_{step}}$, and the explanation time quantiles 25 up to 100, where quantile 100 $q_{100}$ symbols the most expensive step to generate over all the problem instances.

\begin{table}[!h]
	\centering
	\begin{adjustbox}{max width=\textwidth}
		\begin{tabular}{r|c|c|cccccc}
			& \multicolumn{8}{c}{\textbf{\emph{Multi}-\subsetmaxsat}}\\[1ex]
			\textbf{config}          & $\mathbf{\overline{t_{1}}}$  & $\mathbf{\overline{t_{step}}}$ &  $\mathbf{q_{25}}$ &  $\mathbf{q_{50}}$ &  $\mathbf{q_{75}}$  & $\mathbf{q_{95}}$ &  $\mathbf{q_{98}}$ & $\mathbf{q_{100}}$ \\ \midrule
			\comus          &              3.41 &        2.54 &  0.51 &  0.77 &  1.39 &  8.37 & 18.50 &  523.67 \\
			\comusbound     &              1.70 &        1.35 &  0.58 &  0.92 &  1.88 &  3.67 &  4.45 &    7.34 \\
			\ousbiter       &              0.92 &        1.31 &  0.54 &  0.90 &  1.84 &  3.58 &  4.24 &    6.11 \\ \hline
			\comusincr     &              4.33 &        3.66 &  0.27 &  0.42 &  0.72 &  8.69 & 21.46 & 3452.20 \\
			\ousbincr       &              3.22 &        0.58 &  0.39 &  0.51 &  0.67 &  1.07 &  1.57 &   17.58 \\
			\ousbiterincr   &              2.58 &        0.45 &  0.32 &  0.43 &  0.54 &  0.70 &  1.04 &   18.05 \\
		\end{tabular}
	\end{adjustbox}
	\caption{\textbf{Non-timed out instances:} Decomposition of runtime for individual explanations into time to generate the first explanation step ($t_1$), the average time to produce an additional explanation step $\overline{t_{expl}}$ and $q_{xx}$ the quantiles of the explanation times.}
	\label{tab:indiv-explanations-all}
\end{table}

For the non-incremental variants, in the upper part of \cref{tab:indiv-explanations-all}, \comus is able to compute part of the explanations faster than the other \comus configurations. We see, however, that both incremental and non-incremental \comus still take a lot of time for some explanations compared to the other \comus-based algorithms.

In the case of \emph{\comusbound}, the time to the first explanation reflects how important the ordering of literals to explain is for quickly finding a good bound on the cost of the next best explanation. 

\paragraph{Logic Grid Puzzles} In \cref{tab:indiv-explanations-logic-grid-puzzles}, we detail the runtime for explaining the logic grid puzzles of \citet{ecai/BogaertsGCG20}. 
In that work, due to the use of heuristics for finding low-cost explanations, the full explanation of a puzzle took between `15 minutes and a few hours'. 
Explaining the pasta puzzle takes less than 5 minutes using \ousbiterincr with \emph{Multi}-\subsetmaxsat~as \corrsubsets procedure. 
Not only can we generate optimal explanations with respect to a given cost function, but we can also quickly provide an initial explanation and additional explanations.
This suggests that our methods can be integrated into an interactive setting.

\begin{table}[!h]
	\centering
		\begin{tabular}{l|c|c|c}
			\textbf{puzzle} & $\mathbf{t_1}$ & $\mathbf{\overline{t_{expl}}}$ & $\mathbf{t_{tot}}$ \\
			\hline
			origin &1.07s & 0.29s & 43.05s \\
			p12 & 1.25s & 0.43s & 63.85s \\
			p13 & 1.20s & 0.37s & 56.22s \\
			p16 & 1.04s & 0.27s & 40.98s \\
			p18 & 1.14s & 0.16s & 23.32s \\
			p19 & 3.76s & 0.55s & 137.08s \\
			p20 & 1.14s & 0.16s & 23.38s \\
			p25 & 1.12s & 0.26s & 38.32s \\
			p93 & 1.13s & 0.32s & 47.83s \\
			pasta & 0.60s & 2.98s & 286.03s \\
		\end{tabular}
	\caption{Runtime decomposition of \textbf{Logic Grid Puzzles} of \citet{ecai/BogaertsGCG20} into time to generate the first explanation step ($t_1$), the average time to produce an additional explanation step $\overline{t_{expl}}$, and time to generate a full explanation sequence $t_{tot}$.}
\label{tab:indiv-explanations-logic-grid-puzzles}
\end{table}

	\begin{table}[!h]
		\centering
		\begin{adjustbox}{max width=\textwidth} 	
			\begin{tabular}{r|c|c|c}
				\textbf{config} & $\#$none&$\#$timed out  &        $\mathbf{\overline{ \raisebox{0pt}[1.2\height]{{\text{\small \%}} expl}}}$  \\ 
				\midrule
				\comus &   3 & 112 &   44.57 \\
				\comusbound &   32 & 92  & 35.21 \\
				\ousbiter &   15 & 92 &  43.43  \\\hline
				\comusincr &   4 & 158  & 52.94  \\
				\ousbincr &   44 & 94 &  31.98 \\
				\ousbiterincr &   15 & 92  & 49.45  \\
			\end{tabular}
			
		\end{adjustbox}
		\caption{\textbf{Timed out instances:} \emph{$\#$none} is the number of instances where no explanations step was found before the time out, \emph{$\#$timed out} is the number of instances that timed out, $\overline{\raisebox{0pt}[9pt]{{\small \%} expl}}$ is the average percentage of literals explained.}
		\label{tab:timed-out-instances-indiv-explanations-all}
	\end{table}
	
The next step is to find out what causes instances to timeout before the full explanation sequence has been generated with \emph{Multi}-\subsetmaxsat~for all \comus configurations. 

\paragraph{Timed out Instances} Most of the instances that timed out had on average more literals to explain than those that did not.
\cref{tab:timed-out-instances-indiv-explanations-all} provides more context about the timed out instances: 
\begin{itemize}
	\item \emph{$\#$timed out} is the number of instances that have timed out.
	\item \emph{$\#$none} is to the number of instances where no explanation step was generated.
	\item 	$\mathbf{\overline{ \raisebox{0pt}[1.2\height]{{\text{\small \%}} expl}}}$ is the average percentage of literals that were explained before the instance timed out.
\end{itemize}
In \cref{fig:timedout-instances} we specifically report the average time taken to generate a first explaining step ($\mathbf{\overline{t_{1}}}$) for all configurations. 

Note from \cref{tab:timed-out-instances-indiv-explanations-all} that \comus{(+Incr.)}~has the smallest number of instances where no explanation was found. Unlike \comusbound{(+Incr.)}~and \comussplit{(+Incr.)}, \comus{(+Incr.)}~does not need to iterate over the many literals to explain in order to find a good bound on the cost of the next best explanation step. 
Furthermore, we observe that \comus{(+Incr.)}~is the fastest at generating a first explanation step with many outliers close to the .

Similar to the observations in \cref{sec:Instance-level speed-up with incrementality}, incrementality has a high impact on the time to explain an instance for all \comus configurations. 
Adding incrementality to all \comus configurations results in more instances with no explanation found than without incrementality. 
However, incrementality helps to explain more literals, except for \comusbound, which first requires computing many OUSs before the best one is found. 
A similar trend can be seen in \cref{fig:timedout-instances}, where the introduction of incrementality increases the time to generate a first explanation step.
Second, \comusbound depends on a good ordering of the literals, which can change at any explanation step in the sequence.

\begin{figure}[!h]
	\centering
	\includegraphics[width=0.65\textwidth]{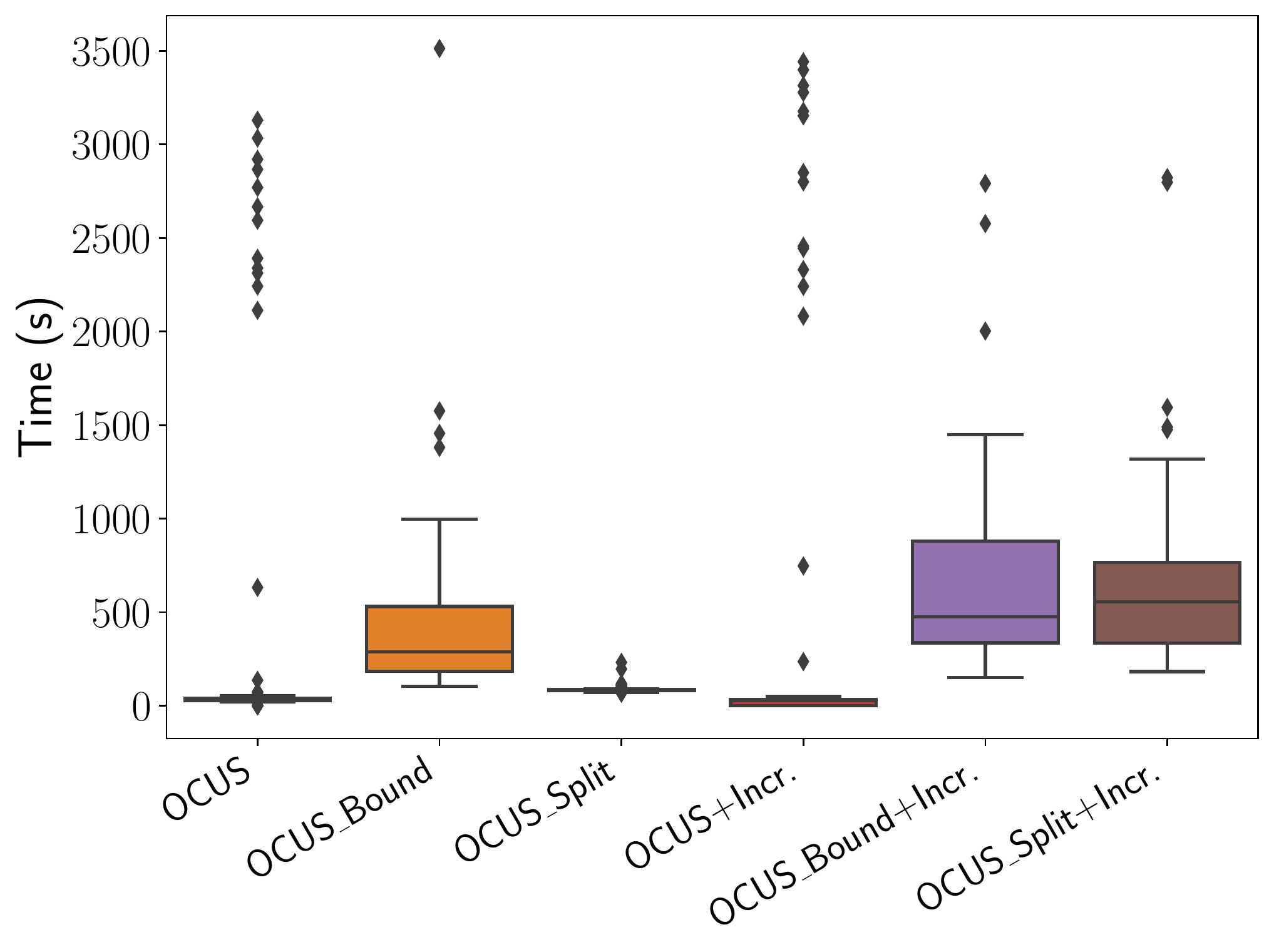}
	\caption{Average time to a first explanation step $\mathbf{\overline{t_{1}}}$.}
	\label{fig:timedout-instances}
\end{figure}

To conclude, for an interactive context where the user asks for \emph{an immediate explanation}, \comussplit will be the fastest to compute an explanation step. 
If the user \emph{explores the sequence of explanations} and repeatedly \emph{asks for additional explanations}, \ousbiterincr~will be able to capitalise on incrementality to bring the average time close to real time.

\section{Conclusion and Future Work}

In this paper, we tackled the problem of efficiently generating explanations that are optimal with respect to a cost function.
For that, we introduced algorithms that replace the many calls to an unsatisfiable core extractor of \citet{bogaerts2021framework}, with a single call to an optimal constrained unsatisfiable subset (\comus).
Our algorithm for computing \comus{}s uses the implicit hitting set duality between Minimum Correction Subsets and Minimal Unsatisfiable Subsets. 
We propose two additional variants for the `exactly one of' constraint used in step-wise explanation generation.
The main bottleneck of these approaches is having to repeatedly compute \emph{optimal} hitting sets.
To compensate, we have developed methods for enumerating correction subsets that explore the trade-off between efficiency and quality of the subsets generated.
An open question in an interactive setting is whether there are other methods of enumerating correction subsets that are better suited to reducing the time to first explanation and the average explanation time.

To efficiently generate the \emph{whole} explanation sequence, we introduced \emph{incrementality}, which allows the reuse of computed information, i.e. satisfiable subsets remain valid from one explanation call to another.
In the case of \comus for explanation sequence generation, where the underlying hitting set solver is MIP-based, we instantiate the MIP solver once for all explanation steps and keep track of computed sets-to-hit.
However, MIP solvers do not natively support non-linear cost functions. 
Therefore, the current implementation of \comus may not be able to fully capture the complexity of what constitutes a good explanation due to the limitations of the cost functions we can encode.
Furthermore, the question of which interpretability metric to use to characterise understandable explanations remains open.  

The concept of (incremental) \comus is not limited to explaining satisfaction problems. We are also interested in exploring other applications, such as configuration and machine learning, where it is necessary to compute not only minimal but also optimal unsatisfiable subsets.
For example, a potential avenue for future work is how these explanation generation methods map to a constraint optimisation setting where branching or searching is required to solve the problem at hand. 
The synergies of our approach with the more general problem of QMaxSAT \cite{DBLP:journals/constraints/IgnatievJM16} is another open question. 

Finally, for interactive settings such as interactive tutoring systems, we are able to generate explanations in near real time (less than a second) using \comussplit. Another question is how to distribute the computation across multiple cores to parallelize the computation to further speed up the generation of explanations.

\acks{This research received partial funding from the Flemish Government (AI Research Program); the FWO Flanders project G070521N; and funding from the European Research Council (ERC) under the European Union’s Horizon 2020 research and innovation program (Grant No. 101002802, CHAT-Opt).}

\appendix
\newpage
\section{Puzzle Data}\label{apx:puzzledata}

\begin{table}[!h]\centering
	\begin{tabular}{llrrr}
		\toprule
		&              &  \multicolumn{3}{c}{\textbf{Properties of instances}} \\
		&              &  n &  avg. $\#$ {clauses} &  avg. $\#$ lits-to-explain\\
		\textbf{puzzle type} & \textbf{instance} &              &          &                      \\
		\midrule
		logic & pasta &    1 &           4572 &                    96 \\
		& 150 &    8 &           7535 &                   150 \\
		& 250 &    1 &          17577 &                   250 \\\hline
		sudoku & 9x9-easy &   25 &          14904 &                   497 \\
		& 9x9-simple &   25 &          14904 &                   497 \\
		& 9x9-intermediate &   25 &          14904 &                   501 \\
		& 9x9-expert &   25 &          14904 &                   505 \\\hline
		demystify & binairo &  156 &          16759 &                    99 \\
		& garam &   10 &          14486 &                   756 \\
		& kakurasu &    1 &            226 &                    16 \\
		& kakuro &    6 &           6961 &                   240 \\
		& killersudoku &   13 &        3609685 &                   729 \\
		& miracle &    1 &          59331 &                   729 \\
		& nonogram &    1 &           9813 &                    36 \\
		& skyscrapers &   16 &          17991 &                   159 \\
		& star-battle &    5 &           7533 &                    82 \\
		& sudoku &   76 &          34143 &                   729 \\
		& tents &    4 &          18644 &                   476 \\
		& thermometer &    1 &           1227 &                    36 \\
		& x-sums &    1 &         109717 &                   729 \\
		\bottomrule
	\end{tabular}
	\caption{Characteristics of puzzle instances}
	\label{tab:puzzle-composition-family}
\end{table}

\bibliography{krrlib,ref}

\begin{thebibliography}{}

\bibitem[\protect\BCAY{Alrabbaa, Baader, Borgwardt, Koopmann,\ \BBA\
  Kovtunova}{Alrabbaa et~al.}{2021}]{alrabbaa2021finding}
Alrabbaa, C., Baader, F., Borgwardt, S., Koopmann, P., \BBA\ Kovtunova, A.
  \BBOP2021\BBCP.
\newblock \BBOQ Finding good proofs for description logic entailments using
  recursive quality measures\BBCQ\
\newblock {\Bem Automated Deduction-CADE}, 12--15.

\bibitem[\protect\BCAY{Baader, Horrocks,\ \BBA\ Sattler}{Baader
  et~al.}{2004}]{baader2004description}
Baader, F., Horrocks, I., \BBA\ Sattler, U. \BBOP2004\BBCP.
\newblock \BBOQ Description logics\BBCQ\
\newblock In {\Bem Handbook on ontologies}, \BPGS\ 3--28. Springer.

\bibitem[\protect\BCAY{Bacchus\ \BBA\ Katsirelos}{Bacchus\ \BBA\
  Katsirelos}{2015}]{bacchus2015using}
Bacchus, F.\BBACOMMA\  \BBA\ Katsirelos, G. \BBOP2015\BBCP.
\newblock \BBOQ Using minimal correction sets to more efficiently compute
  minimal unsatisfiable sets\BBCQ\
\newblock In {\Bem Computer Aided Verification: 27th International Conference,
  CAV 2015, San Francisco, CA, USA, July 18-24, 2015, Proceedings, Part II},
  \BPGS\ 70--86. Springer.

\bibitem[\protect\BCAY{Bacchus\ \BBA\ Katsirelos}{Bacchus\ \BBA\
  Katsirelos}{2016}]{bacchus2016finding}
Bacchus, F.\BBACOMMA\  \BBA\ Katsirelos, G. \BBOP2016\BBCP.
\newblock \BBOQ Finding a collection of muses incrementally\BBCQ\
\newblock In {\Bem Integration of AI and OR Techniques in Constraint
  Programming: 13th International Conference, CPAIOR 2016, Banff, AB, Canada,
  May 29-June 1, 2016, Proceedings 13}, \BPGS\ 35--44. Springer.

\bibitem[\protect\BCAY{Bend{\'\i}k\ \BBA\ {\v{C}}ern{\'a}}{Bend{\'\i}k\ \BBA\
  {\v{C}}ern{\'a}}{2020a}]{bendik2020must}
Bend{\'\i}k, J.\BBACOMMA\  \BBA\ {\v{C}}ern{\'a}, I. \BBOP2020a\BBCP.
\newblock \BBOQ Must: minimal unsatisfiable subsets enumeration tool\BBCQ\
\newblock In {\Bem Tools and Algorithms for the Construction and Analysis of
  Systems: 26th International Conference, TACAS 2020, Held as Part of the
  European Joint Conferences on Theory and Practice of Software, ETAPS 2020,
  Dublin, Ireland, April 25--30, 2020, Proceedings, Part I 26}, \BPGS\
  135--152. Springer.

\bibitem[\protect\BCAY{Bend{\'\i}k\ \BBA\ {\v{C}}ern{\'a}}{Bend{\'\i}k\ \BBA\
  {\v{C}}ern{\'a}}{2020b}]{bendik2020replication}
Bend{\'\i}k, J.\BBACOMMA\  \BBA\ {\v{C}}ern{\'a}, I. \BBOP2020b\BBCP.
\newblock \BBOQ Replication-guided enumeration of minimal unsatisfiable
  subsets\BBCQ\
\newblock In {\Bem Principles and Practice of Constraint Programming: 26th
  International Conference, CP 2020, Louvain-la-Neuve, Belgium, September
  7--11, 2020, Proceedings 26}, \BPGS\ 37--54. Springer.

\bibitem[\protect\BCAY{Biere, Heule, {van Maaren},\ \BBA\ Walsh}{Biere
  et~al.}{2009}]{faia/2009-185}
Biere, A., Heule, M., {van Maaren}, H., \BBA\ Walsh, T. \BBOP2009\BBCP.
\newblock {\Bem Handbook of Satisfiability}.

\bibitem[\protect\BCAY{Bogaerts, Gamba, Claes,\ \BBA\ Guns}{Bogaerts
  et~al.}{2020}]{ecai/BogaertsGCG20}
Bogaerts, B., Gamba, E., Claes, J., \BBA\ Guns, T. \BBOP2020\BBCP.
\newblock \BBOQ Step-wise explanations of constraint satisfaction
  problems\BBCQ\
\newblock In {\Bem Proceedigns of {ECAI}}, \BPGS\ 640--647.

\bibitem[\protect\BCAY{Bogaerts, Gamba,\ \BBA\ Guns}{Bogaerts
  et~al.}{2021}]{bogaerts2021framework}
Bogaerts, B., Gamba, E., \BBA\ Guns, T. \BBOP2021\BBCP.
\newblock \BBOQ A framework for step-wise explaining how to solve constraint
  satisfaction problems\BBCQ\
\newblock {\Bem Artificial Intelligence}, {\Bem 300}, 103550.

\bibitem[\protect\BCAY{Chakraborti, Sreedharan, Zhang,\ \BBA\
  Kambhampati}{Chakraborti et~al.}{2017}]{chakraborti2017plan}
Chakraborti, T., Sreedharan, S., Zhang, Y., \BBA\ Kambhampati, S.
  \BBOP2017\BBCP.
\newblock \BBOQ Plan explanations as model reconciliation: moving beyond
  explanation as soliloquy\BBCQ\
\newblock In {\Bem Proceedings of {IJCAI}}, \BPGS\ 156--163.

\bibitem[\protect\BCAY{{\v{C}}yras, Letsios, Misener,\ \BBA\ Toni}{{\v{C}}yras
  et~al.}{2019}]{vcyras2019argumentation}
{\v{C}}yras, K., Letsios, D., Misener, R., \BBA\ Toni, F. \BBOP2019\BBCP.
\newblock \BBOQ Argumentation for explainable scheduling\BBCQ\
\newblock In {\Bem Proceedings of {AAAI}}, \BPGS\ 2752--2759.

\bibitem[\protect\BCAY{Davies\ \BBA\ Bacchus}{Davies\ \BBA\
  Bacchus}{2013}]{DBLP:conf/sat/DaviesB13}
Davies, J.\BBACOMMA\  \BBA\ Bacchus, F. \BBOP2013\BBCP.
\newblock \BBOQ Exploiting the power of {MIP} solvers in {MAXsat}\BBCQ\
\newblock In {\Bem Proceedings of {SAT}}, \BPGS\ 166--181.

\bibitem[\protect\BCAY{Dershowitz, Hanna,\ \BBA\ Nadel}{Dershowitz
  et~al.}{2006}]{DBLP:conf/sat/DershowitzHN06}
Dershowitz, N., Hanna, Z., \BBA\ Nadel, A. \BBOP2006\BBCP.
\newblock \BBOQ A scalable algorithm for minimal unsatisfiable core
  extraction\BBCQ\
\newblock In {\Bem Proceedings of {SAT}}, \BPGS\ 36--41.

\bibitem[\protect\BCAY{Espasa, Gent, Hoffmann, Jefferson,\ \BBA\ Lynch}{Espasa
  et~al.}{2021}]{schotten}
Espasa, J., Gent, I.~P., Hoffmann, R., Jefferson, C., \BBA\ Lynch, A.~M.
  \BBOP2021\BBCP.
\newblock \BBOQ Using small muses to explain how to solve pen and paper
  puzzles\BBCQ\
\newblock {\Bem ArXiv}, {\Bem abs/2104.15040}.

\bibitem[\protect\BCAY{Robustness\ \BBA\
  of~Artificial~Intelligence}{{FET}}{2019}]{fetproact}
{FET} \BBOP2019\BBCP.
\newblock \BBOQ Fetproact-eic-05-2019, fet proactive: emerging paradigms and
  communities, call\BBCQ.
\newblock Horizon 2020 Framework Programme.

\bibitem[\protect\BCAY{Fox, Long,\ \BBA\ Magazzeni}{Fox
  et~al.}{2017}]{fox2017explainable}
Fox, M., Long, D., \BBA\ Magazzeni, D. \BBOP2017\BBCP.
\newblock \BBOQ Explainable planning\BBCQ\
\newblock In {\Bem Proceedings of {IJCAI'17-XAI}}.

\bibitem[\protect\BCAY{Freuder, Likitvivatanavong,\ \BBA\ Wallace}{Freuder
  et~al.}{2001}]{freuder2001explanation}
Freuder, E.~C., Likitvivatanavong, C., \BBA\ Wallace, R.~J. \BBOP2001\BBCP.
\newblock \BBOQ Explanation and implication for configuration problems\BBCQ\
\newblock In {\Bem IJCAI 2001 workshop on configuration}, \BPGS\ 31--37.

\bibitem[\protect\BCAY{Frisch, Harvey, Jefferson, Martinez-Hernandez,\ \BBA\
  Miguel}{Frisch et~al.}{2008}]{frisch2008essence}
Frisch, A.~M., Harvey, W., Jefferson, C., Martinez-Hernandez, B., \BBA\ Miguel,
  I. \BBOP2008\BBCP.
\newblock \BBOQ Essence: A constraint language for specifying combinatorial
  problems\BBCQ\
\newblock {\Bem Constraints}, {\Bem 13\/}(3), 268--306.

\bibitem[\protect\BCAY{Gamba, Bogaerts,\ \BBA\ Guns}{Gamba
  et~al.}{2021}]{ijcai2021}
Gamba, E., Bogaerts, B., \BBA\ Guns, T. \BBOP2021\BBCP.
\newblock \BBOQ Efficiently explaining {CSPs} with unsatisfiable subset
  optimization\BBCQ\
\newblock In Zhou, Z.-H.\BED, {\Bem Proceedings of the Thirtieth International
  Joint Conference on Artificial Intelligence, {IJCAI-21}}, \BPGS\ 1381--1388.
  International Joint Conferences on Artificial Intelligence Organization.
\newblock Main Track.

\bibitem[\protect\BCAY{Gershman, Koifman,\ \BBA\ Strichman}{Gershman
  et~al.}{2008}]{DBLP:journals/fmsd/GershmanKS08}
Gershman, R., Koifman, M., \BBA\ Strichman, O. \BBOP2008\BBCP.
\newblock \BBOQ An approach for extracting a small unsatisfiable core\BBCQ\
\newblock {\Bem Formal Methods in System Design}, {\Bem 33\/}(1-3), 1--27.

\bibitem[\protect\BCAY{Goldberg\ \BBA\ Novikov}{Goldberg\ \BBA\
  Novikov}{2003}]{goldberg}
Goldberg, E.\BBACOMMA\  \BBA\ Novikov, Y. \BBOP2003\BBCP.
\newblock \BBOQ Verification of proofs of unsatisfiability for {CNF}
  formulas\BBCQ\
\newblock In {\Bem Proceedings of {DATE}}, \BPGS\ 10886--10891.

\bibitem[\protect\BCAY{Guidotti, Monreale, Ruggieri, Turini, Giannotti,\ \BBA\
  Pedreschi}{Guidotti et~al.}{2018}]{guidotti2018survey}
Guidotti, R., Monreale, A., Ruggieri, S., Turini, F., Giannotti, F., \BBA\
  Pedreschi, D. \BBOP2018\BBCP.
\newblock \BBOQ A survey of methods for explaining black box models\BBCQ\
\newblock {\Bem ACM computing surveys (CSUR)}, {\Bem 51\/}(5), 1--42.

\bibitem[\protect\BCAY{Gunning}{Gunning}{2017}]{gunning2017explainable}
Gunning, D. \BBOP2017\BBCP.
\newblock \BBOQ Explainable artificial intelligence {(XAI)}\BBCQ\
\newblock {\Bem Defense Advanced Research Projects Agency}, {\Bem 2}.

\bibitem[\protect\BCAY{Hamon, Junklewitz,\ \BBA\ Sanchez}{Hamon
  et~al.}{2020}]{hamonrobustness}
Hamon, R., Junklewitz, H., \BBA\ Sanchez, I. \BBOP2020\BBCP.
\newblock \BBOQ Robustness and explainability of artificial intelligence\BBCQ\
\newblock {\Bem Publications Office of the European Union}.

\bibitem[\protect\BCAY{Hildebrandt, Castillo, Celis, Ruggieri, Taylor,\ \BBA\
  Zanfir{-}Fortuna}{Hildebrandt et~al.}{2020}]{FAT}
Hildebrandt, M., Castillo, C., Celis, E., Ruggieri, S., Taylor, L., \BBA\
  Zanfir{-}Fortuna, G.\BEDS. \BBOP2020\BBCP.
\newblock {\Bem Proceedings of {FAT*}}.

\bibitem[\protect\BCAY{Huang}{Huang}{2005}]{huang}
Huang, J. \BBOP2005\BBCP.
\newblock \BBOQ Mup: A minimal unsatisfiability prover\BBCQ\
\newblock In {\Bem Proceedings of the Asia and South Pacific Design Automation
  Conference, ASP-DAC}, \BPGS\ 432-- 437 Vol. 1.

\bibitem[\protect\BCAY{Ignatiev, Janota,\ \BBA\ Marques{-}Silva}{Ignatiev
  et~al.}{2016}]{DBLP:journals/constraints/IgnatievJM16}
Ignatiev, A., Janota, M., \BBA\ Marques{-}Silva, J. \BBOP2016\BBCP.
\newblock \BBOQ Quantified maximum satisfiability\BBCQ\
\newblock {\Bem Constraints}, {\Bem 21\/}(2), 277--302.

\bibitem[\protect\BCAY{Ignatiev, Morgado,\ \BBA\ Marques{-}Silva}{Ignatiev
  et~al.}{2018}]{pysat}
Ignatiev, A., Morgado, A., \BBA\ Marques{-}Silva, J. \BBOP2018\BBCP.
\newblock \BBOQ {PySAT:} {A} {Python} toolkit for prototyping with {SAT}
  oracles\BBCQ\
\newblock In {\Bem SAT}, \BPGS\ 428--437.

\bibitem[\protect\BCAY{Ignatiev, Narodytska,\ \BBA\ Marques-Silva}{Ignatiev
  et~al.}{2019}]{ignatiev2019abduction}
Ignatiev, A., Narodytska, N., \BBA\ Marques-Silva, J. \BBOP2019\BBCP.
\newblock \BBOQ Abduction-based explanations for machine learning models\BBCQ\
\newblock In {\Bem Proceedings of {AAAI}}, \BPGS\ 1511--1519.

\bibitem[\protect\BCAY{Ignatiev, Previti, Liffiton,\ \BBA\
  Marques-Silva}{Ignatiev et~al.}{2015}]{ignatiev2015smallest}
Ignatiev, A., Previti, A., Liffiton, M., \BBA\ Marques-Silva, J.
  \BBOP2015\BBCP.
\newblock \BBOQ Smallest {MUS} extraction with minimal hitting set
  dualization\BBCQ\
\newblock In {\Bem Proceedings of {CP}}.

\bibitem[\protect\BCAY{Junker}{Junker}{2001}]{junker2001quickxplain}
Junker, U. \BBOP2001\BBCP.
\newblock \BBOQ {QuickXPlain}: Conflict detection for arbitrary constraint
  propagation algorithms\BBCQ\
\newblock In {\Bem IJCAI'01 Workshop on Modelling and Solving problems with
  constraints}.

\bibitem[\protect\BCAY{Kleine{ }B{\"{u}}ning\ \BBA\ Bubeck}{Kleine{
  }B{\"{u}}ning\ \BBA\ Bubeck}{2009}]{QBF}
Kleine{ }B{\"{u}}ning, H.\BBACOMMA\  \BBA\ Bubeck, U. \BBOP2009\BBCP.
\newblock \BBOQ Theory of quantified boolean formulas\BBCQ\
\newblock In {\Bem Handbook of Satisfiability}, \BPGS\ 735--760.

\bibitem[\protect\BCAY{Koopmann}{Koopmann}{2021}]{koopmann2021two}
Koopmann, P. \BBOP2021\BBCP.
\newblock \BBOQ Two ways of explaining negative entailments in description
  logics using abduction\BBCQ\
\newblock {\Bem Explainable Logic-Based Knowledge Representation (XLoKR 2021)}.

\bibitem[\protect\BCAY{Langley, Meadows, Sridharan,\ \BBA\ Choi}{Langley
  et~al.}{2017}]{langley2017explainable}
Langley, P., Meadows, B., Sridharan, M., \BBA\ Choi, D. \BBOP2017\BBCP.
\newblock \BBOQ Explainable agency for intelligent autonomous systems\BBCQ\
\newblock In {\Bem Twenty-Ninth IAAI Conference}.

\bibitem[\protect\BCAY{Li\ \BBA\ Many{\`{a}}}{Li\ \BBA\
  Many{\`{a}}}{2021}]{DBLP:series/faia/LiM09}
Li, C.~M.\BBACOMMA\  \BBA\ Many{\`{a}}, F. \BBOP2021\BBCP.
\newblock \BBOQ {MaxSAT}, hard and soft constraints\BBCQ\
\newblock In {\Bem Handbook of satisfiability}, \BPGS\ 903--927. IOS Press.

\bibitem[\protect\BCAY{Liao\ \BBA\ Van Der~Torre}{Liao\ \BBA\ Van
  Der~Torre}{2020}]{liao2020explanation}
Liao, B.\BBACOMMA\  \BBA\ Van Der~Torre, L. \BBOP2020\BBCP.
\newblock \BBOQ Explanation semantics for abstract argumentation\BBCQ\
\newblock In {\Bem Computational Models of Argument}, \BPGS\ 271--282. IOS
  Press.

\bibitem[\protect\BCAY{Liffiton, Previti, Malik,\ \BBA\ Marques-Silva}{Liffiton
  et~al.}{2016}]{liffiton2016fast}
Liffiton, M.~H., Previti, A., Malik, A., \BBA\ Marques-Silva, J.
  \BBOP2016\BBCP.
\newblock \BBOQ Fast, flexible mus enumeration\BBCQ\
\newblock {\Bem Constraints}, {\Bem 21\/}(2), 223--250.

\bibitem[\protect\BCAY{Liffiton\ \BBA\ Sakallah}{Liffiton\ \BBA\
  Sakallah}{2008}]{DBLP:journals/jar/LiffitonS08}
Liffiton, M.~H.\BBACOMMA\  \BBA\ Sakallah, K.~A. \BBOP2008\BBCP.
\newblock \BBOQ Algorithms for computing minimal unsatisfiable subsets of
  constraints\BBCQ\
\newblock {\Bem J. Autom. Reasoning}, {\Bem 40\/}(1), 1--33.

\bibitem[\protect\BCAY{Lundberg\ \BBA\ Lee}{Lundberg\ \BBA\
  Lee}{2017}]{lundberg2017unified}
Lundberg, S.~M.\BBACOMMA\  \BBA\ Lee, S.-I. \BBOP2017\BBCP.
\newblock \BBOQ A unified approach to interpreting model predictions\BBCQ\
\newblock In {\Bem Proceedings of {NIPS}}, \BPGS\ 4765--4774.

\bibitem[\protect\BCAY{Lynce\ \BBA\ Silva}{Lynce\ \BBA\
  Silva}{2004}]{conf/sat/LynceM04}
Lynce, I.\BBACOMMA\  \BBA\ Silva, J. P.~M. \BBOP2004\BBCP.
\newblock \BBOQ On computing minimum unsatisfiable cores\BBCQ\
\newblock In {\Bem Proceedings of {SAT}}.

\bibitem[\protect\BCAY{Marques-Silva}{Marques-Silva}{2010}]{marques2010minimal}
Marques-Silva, J. \BBOP2010\BBCP.
\newblock \BBOQ Minimal unsatisfiability: Models, algorithms and
  applications\BBCQ\
\newblock In {\Bem 2010 40th IEEE International Symposium on Multiple-Valued
  Logic}.

\bibitem[\protect\BCAY{Marques-Silva, Heras, Janota, Previti,\ \BBA\
  Belov}{Marques-Silva et~al.}{2013}]{marques2013computing}
Marques-Silva, J., Heras, F., Janota, M., Previti, A., \BBA\ Belov, A.
  \BBOP2013\BBCP.
\newblock \BBOQ On computing minimal correction subsets\BBCQ\
\newblock In {\Bem Twenty-Third International Joint Conference on Artificial
  Intelligence}.

\bibitem[\protect\BCAY{Miller}{Miller}{2019}]{miller2019explanation}
Miller, T. \BBOP2019\BBCP.
\newblock \BBOQ Explanation in artificial intelligence: Insights from the
  social sciences\BBCQ\
\newblock {\Bem Artificial Intelligence}, {\Bem 267}, 1--38.

\bibitem[\protect\BCAY{Miller, Weber,\ \BBA\ Magazzeni}{Miller
  et~al.}{2019}]{xai-ijcai}
Miller, T., Weber, R., \BBA\ Magazzeni, D.\BEDS. \BBOP2019\BBCP.
\newblock {\Bem Proceedings of the IJCAI 2019 Workshop on Explainable
  Artificial Intelligence}.

\bibitem[\protect\BCAY{Modgil, Toni, Bex, Bratko, Chesnevar, Dvo{\v{r}}{\'a}k,
  Falappa, Fan, Gaggl, Garc{\'\i}a, et~al.}{Modgil
  et~al.}{2013}]{modgil2013added}
Modgil, S., Toni, F., Bex, F., Bratko, I., Chesnevar, C.~I., Dvo{\v{r}}{\'a}k,
  W., Falappa, M.~A., Fan, X., Gaggl, S.~A., Garc{\'\i}a, A.~J., et~al.
  \BBOP2013\BBCP.
\newblock \BBOQ The added value of argumentation\BBCQ\
\newblock In {\Bem Agreement technologies}, \BPGS\ 357--403. Springer.

\bibitem[\protect\BCAY{Nightingale, Akg{\"u}n, Gent, Jefferson, Miguel,\ \BBA\
  Spracklen}{Nightingale et~al.}{2017}]{nightingale2017automatically}
Nightingale, P., Akg{\"u}n, {\"O}., Gent, I.~P., Jefferson, C., Miguel, I.,
  \BBA\ Spracklen, P. \BBOP2017\BBCP.
\newblock \BBOQ Automatically improving constraint models in savile row\BBCQ\
\newblock {\Bem Artificial Intelligence}, {\Bem 251}, 35--61.

\bibitem[\protect\BCAY{Obeid, Obeid, Moubaiddin,\ \BBA\ Obeid}{Obeid
  et~al.}{2019}]{obeid2019using}
Obeid, M., Obeid, Z., Moubaiddin, A., \BBA\ Obeid, N. \BBOP2019\BBCP.
\newblock \BBOQ Using description logic and abox abduction to capture medical
  diagnosis\BBCQ\
\newblock In {\Bem International Conference on Industrial, Engineering and
  Other Applications of Applied Intelligent Systems}, \BPGS\ 376--388.
  Springer.

\bibitem[\protect\BCAY{Oh, Mneimneh, Andraus, Sakallah,\ \BBA\ Markov}{Oh
  et~al.}{2004}]{DBLP:conf/dac/OhMASM04}
Oh, Y., Mneimneh, M.~N., Andraus, Z.~S., Sakallah, K.~A., \BBA\ Markov, I.~L.
  \BBOP2004\BBCP.
\newblock \BBOQ {AMUSE:} a minimally-unsatisfiable subformula extractor\BBCQ\
\newblock In {\Bem Proceedings of {DAC}}, \BPGS\ 518--523.

\bibitem[\protect\BCAY{Reiter}{Reiter}{1987}]{ai/Reiter87}
Reiter, R. \BBOP1987\BBCP.
\newblock \BBOQ A theory of diagnosis from first principles\BBCQ\
\newblock {\Bem {AIJ}}, {\Bem 32\/}(1), 57--95.

\bibitem[\protect\BCAY{Rossi, van Beek,\ \BBA\ Walsh}{Rossi
  et~al.}{2006}]{fai/Rossi06}
Rossi, F., van Beek, P., \BBA\ Walsh, T.\BEDS. \BBOP2006\BBCP.
\newblock {\Bem Handbook of Constraint Programming}, \lowercase{\BVOL}~2 of
  {\Bem Foundations of Artificial Intelligence}.
\newblock Elsevier.

\bibitem[\protect\BCAY{Saikko, Wallner,\ \BBA\ J{\"{a}}rvisalo}{Saikko
  et~al.}{2016}]{DBLP:conf/kr/SaikkoWJ16}
Saikko, P., Wallner, J.~P., \BBA\ J{\"{a}}rvisalo, M. \BBOP2016\BBCP.
\newblock \BBOQ Implicit hitting set algorithms for reasoning beyond {NP}\BBCQ\
\newblock In {\Bem Proceedings of {KR}}, \BPGS\ 104--113.

\bibitem[\protect\BCAY{{\v{S}}e{\v{s}}elja\ \BBA\
  Stra{\ss}er}{{\v{S}}e{\v{s}}elja\ \BBA\
  Stra{\ss}er}{2013}]{vsevselja2013abstract}
{\v{S}}e{\v{s}}elja, D.\BBACOMMA\  \BBA\ Stra{\ss}er, C. \BBOP2013\BBCP.
\newblock \BBOQ Abstract argumentation and explanation applied to scientific
  debates\BBCQ\
\newblock {\Bem Synthese}, {\Bem 190\/}(12), 2195--2217.

\bibitem[\protect\BCAY{Sqalli\ \BBA\ Freuder}{Sqalli\ \BBA\
  Freuder}{1996}]{sqalli1996inference}
Sqalli, M.~H.\BBACOMMA\  \BBA\ Freuder, E.~C. \BBOP1996\BBCP.
\newblock \BBOQ Inference-based constraint satisfaction supports
  explanation\BBCQ\
\newblock In {\Bem AAAI/IAAI, Vol. 1}, \BPGS\ 318--325.

\bibitem[\protect\BCAY{Ulbricht\ \BBA\ Wallner}{Ulbricht\ \BBA\
  Wallner}{2021}]{ulbricht2021strong}
Ulbricht, M.\BBACOMMA\  \BBA\ Wallner, J.~P. \BBOP2021\BBCP.
\newblock \BBOQ Strong explanations in abstract argumentation\BBCQ\
\newblock In {\Bem Proceedings of the AAAI Conference on Artificial
  Intelligence}, \lowercase{\BVOL}~35, \BPGS\ 6496--6504.

\bibitem[\protect\BCAY{Vassiliades, Bassiliades,\ \BBA\ Patkos}{Vassiliades
  et~al.}{2021}]{vassiliades2021argumentation}
Vassiliades, A., Bassiliades, N., \BBA\ Patkos, T. \BBOP2021\BBCP.
\newblock \BBOQ Argumentation and explainable artificial intelligence: a
  survey\BBCQ\
\newblock {\Bem The Knowledge Engineering Review}, {\Bem 36}.

\end{thebibliography}
\bibliographystyle{theapa}

\end{document}